\documentclass[11pt]{article}
\usepackage{format}
\usepackage{macros}

\usepackage[utf8]{inputenc}
\usepackage[T1]{fontenc}
\usepackage{microtype}
\usepackage{csquotes}
\usepackage[osf,sc]{mathpazo}
\RequirePackage[scaled=0.90]{helvet}
\RequirePackage[scaled=0.85]{beramono}
\RequirePackage{textcomp}
\definecolor{DarkGreen}{rgb}{0.1,0.5,0.1}
\definecolor{DarkRed}{rgb}{0.5,0.1,0.1}
\definecolor{DarkBlue}{rgb}{0.1,0.1,0.5}
\definecolor{Gray}{rgb}{0.2,0.2,0.2}
\usepackage[pdftex]{hyperref}
\hypersetup{
    unicode=false,          
    pdftoolbar=true,        
    pdfmenubar=true,        
    pdffitwindow=false,      
    pdfnewwindow=true,      
    colorlinks=true,       
    linkcolor=DarkBlue,          
    citecolor=DarkGreen,        
    filecolor=DarkRed,      
    urlcolor=DarkBlue,          
    pdftitle={Backward baselines: Is your model predicting the past?},
    pdfauthor={Moritz Hardt and Michael Kim},
}

\title{Is Your Model Predicting the Past?}
\author{Moritz Hardt\footnote{Max Planck Institute for Intelligent Systems, T\"ubingen, and T\"ubingen AI Center} \and Michael P. Kim\footnote{Miller Institute for Basic Research in Science, University of California, Berkeley}}
\date{}

\begin{document}
\maketitle

\begin{abstract}
When does a machine learning model predict the future of individuals and when does it recite patterns that predate the individuals? In this work, we propose a distinction between these two pathways of prediction, supported by theoretical, empirical, and normative arguments. At the center of our proposal is a family of simple and efficient statistical tests, called \emph{backward baselines}, that demonstrate if, and to what extent, a model recounts the past. Our statistical theory provides guidance for interpreting backward baselines, establishing equivalences between different baselines and familiar statistical concepts.
Concretely, we derive a meaningful backward baseline for auditing a prediction system as a black box, given only background variables and the system's predictions. Empirically, we evaluate the framework on different prediction tasks derived from longitudinal panel surveys, demonstrating the ease and effectiveness of incorporating backward baselines into the practice of machine learning.
\end{abstract}

\section{Introduction}
\label{sec:intro}
Proponents of predictive technologies for consequential decision-making emphasize the seeming ability of statistical models to anticipate individual actions. The ability to predict the future, so the argument goes, creates a rationale for adopting machine learning as policy:
if a risk score charted the future trajectory of individuals, then intervening in a person's life on the basis of the risk score would be justified~\cite{kleinberg2015prediction, obermeyer2016predicting}. At the same time, critical scholars caution that predictive technologies reproduce historical patterns of injustice and social stratification. In this account, rather than predicting future outcomes, statistical risk assessment tools punish individuals for factors predating their own agency~\cite{eubanks2018automating, benjamin2019race}.

Does a statistical model predict individual agency or recite the past? The answer to the question is often not obvious. Consider the problem of loan default prediction, one of many tasks often framed as predicting future outcomes. One predictor might identify individual behavior detrimental to loan repayment and adjust the predicted likelihood of default accordingly.
Another predictor might rely on historical associations between repayment and demographic factors, then predict based solely on the historical factors. Even if the models achieve the same accuracy, they derive predictive power along distinct pathways. In one solution, we rely on the effects of individual behavior on future outcomes. In the other, we reproduce patterns from the past that were determined before, and independently of, individual behavior. The latter form of prediction---resembling a kind of stereotyping---is core to many documented examples of bias and unfairness in the use of machine learning \cite{fta,angwin2016machine,chouldechova2017fair,hkrr,buolamwini2018gender, bolukbasi2016man, caliskan2017semantics}.

The distinction we draw is fundamental to the theory of equality of opportunity. Dworkin partitions attributes of an individual into factors for which the individual is responsible and factors outside the individual's control \cite{dworkin2018welfare, dworkin2018resources}. Similarly, Roemer distinguishes between effort that an individual takes and the individual’s \emph{type}. A type groups individuals of the same \emph{circumstances}, where “circumstances are those aspects of one’s environment (including, perhaps, one’s biological characteristics) which are beyond one’s control” \cite{roemer2000equality, roemer2002equality, roemer2016equality}. Dworkin and Roemer build on this fundamental moral distinction to define what it means to achieve equality in the allocation of resources and opportunity. Here, we focus on the consequences of the same distinction within the context of prediction.

Although the precise distinction is more subtle, we can approximate it with the help of time. Background variables in a prediction problem are those that were determined before the individual, such as, place and date of birth, or parents’ educational attainment. Background variables generally influence both an individual’s actions and the target of prediction. Individual factors are variables that the individual can exert direct---possibly not full---control over. Correspondingly, we coin the term \emph{backward prediction} to describe the use of background variables in prediction, and we use \emph{forward prediction} to refer to the use of individual factors.

\paragraph{Our contribution.}
In this work we formalize the distinction between forward and backward prediction. We build a theory for forward and backward prediction around a family of simple and effective statistical tests, we call \emph{backward baselines}. Backward baselines quantify how much of a predictor’s strength should be attributed to a given set of background variables. Applying our tools, we empirically find that in representative prediction problems involving longitudinal panel data, backward prediction contributes significantly to the strength of the predictor.  

The strength of backward prediction has important consequences. When prediction draws on background factors primarily, it is misleading to interpret the predictor as an individualized risk score. After all, backward predictors are invariant under individual variation and only depend on the individual’s background. The strength of backward prediction speaks to the social and environmental constitution of the target of prediction. Consequently, there is no relative advantage to targeting interventions at the individual level on the basis of backward predictors. Even if individual-level interventions are helpful, there is no added benefit in targeting them based on individual predictions compared with targeting based on background variables.

To give an example, we consider predicting an individual’s year-round medical expenditure based on the longitudinal MEPS panel survey data. We find that a predictor trained only on background variables nearly matches the predictive performance of classifiers trained on all features. Our finding echoes scholarship about the social determinants of health and medical expenditure~\cite{krieger2011epidemiology}.

We envision that backward baselines will form a useful component of the machine learning evaluation toolkit. Straightforward to apply, backward baselines provide valuable insights into the interpretation and validity of prediction in consequential settings.

\paragraph{Predicting the past.}
To introduce our discussion of backward prediction, we consider an explicit data generating process that moves through time. In Figure~\ref{fig:causal-nature}, we depict the temporal dynamics, in the form of a causal graph, with time evolving from left to right. We think of $X$ as individual-level covariates measured today, and $Y$ as an outcome of interest to be measured in the future.
In addition to the standard supervised learning variables, we also model
an additional \emph{context} variable~$W$---predating the measurement of the covariates or outcome---that may directly influence both~$X$ and~$Y$.
Concretely,~$X$ could represent a record of an individual's educational, personal, and financial history, used to predict income~$Y$ measured in~$10$ years, and $W$ could represent specific demographic features from the past, like childhood household income.

This explicit temporal model elucidates the distinction between \emph{forward} prediction and \emph{backward} prediction. Forward predictors model how the present measurements $X$ causally effect the future outcome~$Y$. Backward predictors estimate the outcome by first inferring the past context $W$ from $X$, then predicting $Y$ based on $W$. In other words, backward prediction provides information about $Y$ that could equally be explained by the past context $W$.

\begin{figure}[h]
    \centering
\begin{subfigure}{0.5\textwidth}
\centering
\begin{tikzpicture}[scale=0.15]
\tikzstyle{every node}+=[inner sep=0pt]
\draw [black] (30.8,-27.5) circle (3);
\draw (30.8,-27.5) node {$W$};
\draw [black] (38.8,-18.8) circle (3);
\draw (38.8,-18.8) node {$X$};
\draw [black] (46.2,-27.5) circle (3);
\draw (46.2,-27.5) node {$Y$};
\draw [black] (32.83,-25.29) -- (36.77,-21.01);
\fill [black] (36.77,-21.01) -- (35.86,-21.26) -- (36.6,-21.94); 
\draw [black] (40.74,-21.09) -- (44.26,-25.21);
\fill [black] (44.26,-25.21) -- (44.12,-24.28) -- (43.36,-24.93); 
\draw [black] (33.8,-27.5) -- (43.2,-27.5);
\fill [black] (43.2,-27.5) -- (42.4,-27) -- (42.4,-28); 
\end{tikzpicture}
\end{subfigure}
\begin{subfigure}{0.3\textwidth}
\centering
\parbox{\textwidth}{
\noindent Generating $\Dist$:
\begin{enumerate}
    \item $W \sim \Dist_W$
    \item $X \sim \Dist_{X\vert W}$
    \item $Y \sim \Dist_{Y \vert X,W}$
\end{enumerate}
}
\end{subfigure}
    \caption{Example data generating process for covariates $X$, outcome $Y$, and context $W$.
    Time starts from the left with context $W$ and evolves forward to the right, realizing $X$ then $Y$.}
    \label{fig:causal-nature}
\end{figure}

\paragraph{Backward baselines.}
Machine learning practitioners often build models using any and every predictive pathway available, including the backward pathway. Our goal is to elucidate and disentangle the prediction pathways that a given predictor uses.
Backward baselines provide a careful accounting of the predictor's use of the forward and backward predictive pathways.
The baselines are lightweight to run, only requiring input-output access to the predictive model, and are built on simple, but rigorous statistical foundations.
For instance, a key challenge in reasoning about backward prediction is that the context $W$ is typically robustly encoded within an individual's covariates $X$.
That is, even if we explicitly censor the attributes defining the context, backward prediction from $X$ may still be possible. Backward baselines handle this statistical subtlety gracefully, providing guaranteed estimates of the forward and backward predictive power, regardless of how redundantly $W$ is encoded in $X$.

Our work establishes backward baselines as an effective tool for investigating predictive models.  Our perspective is \emph{not} that the backward prediction pathway is inherently problematic. Rather, we advocate that investigators use backward baselines to understand and contextualize performance numbers in prediction tasks. Adding the baselines to the ``report card'' for supervised learning would add clarity about the underlying mechanisms used to predict. This clarity, in turn, may inform debate about whether machine learning is an appropriate tool for the task at hand. If model builders cannot find a predictor that improves significantly over backward baselines, we should hesitate before turning prediction into policy.

\section{Backward baselines}
\label{sec:baselines}
We work over a data universe $\Xcal\times \Ycal,$ where $\Xcal$ is a feature space and $\Ycal$ is a discrete set of labels in the case of classification problems. For regression problems, we take $\Ycal$ to be the real line~$\Rbb$. Fixing a loss function~$\ell:\Ycal \times \Ycal \to \Rbb^+$, for a given predictor~$h:\Xcal \to \Ycal$, we measure the fit of the predictor in terms of its expected loss over a distribution $(X,Y)\sim\Dist$ supported on $\Xcal\times\Ycal$:
\begin{gather*}
\ell_\Dist(Y,h(X)) = \E[\ell(Y,h(X))]
\end{gather*}
Throughout, we assume $\ell$ is symmetric in its two arguments.
We study both binary classification and regression, focusing on the zero-one loss $\Pr[Y \neq h(X)]$ with $\Ycal=\{0,1\}$ and squared loss $\E[(Y-h(X))^2]$ where $\Ycal=[0,1]$, respectively.

We extend this standard setup with a random variable~$W$, jointly distributed with $(X,Y)$ and supported on a discrete domain $\Wcal$.
The variable~$W$ represents a \emph{context} of both the individual covariates and the outcome of interest. 
While we model them as separate random variables, at times, we assume that $X$ encodes $W$, explicitly or implicitly.
For instance, in Proposition~\ref{prop:basic}(a), we assume that perfect reconstruction of $W$ is statistically possible from $X$.

\paragraph{Backward prediction baseline.}
In our typical story of backward prediction from $X$, we imagine that a predictor first resolves $W$ from $X$, then predicts $Y$ from $W$.
As such, if we are concerned that a predictor~$h$ is using the backward pathway, a natural baseline to compare against is predicting~$Y$ directly from the context $W$.
Fixing a loss $\ell$, we take $\gs:\Wcal \to \Ycal$ to be the statistically optimal predictor of $Y$ from $W$, and consider the following \emph{backward prediction baseline}, $\ell_\Dist(Y,\gs(W))$.
\begin{align*}
    \gs(w) = \argmin_{\yhat \in \Ycal}~ \E[\ell(Y,\yhat) \vert W = w]&&
    \ell_\Dist(Y,\gs(W)) = \E[\ell(Y,\gs(W))]
\end{align*}
The loss $\ell_\Dist(Y,\gs(W))$ provides a fundamental baseline for how predictable the outcome $Y$ is from $W$.
By comparing this baseline to $\ell_\Dist(Y,h(X))$, we can better contextualize the quality of predictions $h$ produces.
In particular, if $h$ does not achieve significantly better loss than $\gs$, then $h$ is not a very impressive predictor: rather than using machine learning to make decisions, you could get the same performance simply by stereotyping based on $W$.

\paragraph{Backward rounding baseline.}
While the optimal backward predictor $\gs$ is fundamental, it only depends on the underlying relationship between~$W$ and~$Y$, and does not depend on any predictor from $X$. Given such a predictor $h:\Xcal \to \Ycal$, we may instead consider a baseline based on prediction of $h(X)$ from $W$. We consider the \emph{backward rounding baseline}, defined by $\gh:\Wcal \to \Ycal$, which we take to be the optimal predictor of $h(X)$ from $W$.
\begin{align*}
    \gh(w) = \argmin_{\yhat \in \Ycal}~ \E[\ell(h(X),\yhat) \vert W = w]&&
    \ell_\Dist(h(X),\gh(W)) = \E[\ell(h(X),\gh(W))]
\end{align*}
Intuitively speaking, if the prediction~$h(X)$ is itself predictable from $W$, then it seems $h$ must be using the backward pathway.
Contrapositively, if $h$ is a forward predictor, then $h(X)$ cannot be predicted from $W$. An interesting aspect of this baseline is that $\gh$ can be estimated even when true outcomes are unavailable, unobserved, or unreliable.
Moreover, in settings where predictions are performative, in the sense of influencing the distribution on outcomes~\cite{perdomo2020performative}, the backward prediction baseline may not be applicable, while the backward rounding baseline is unaffected.

To understand what these two baselines measure exactly and how they relate, we need to formally define backward and forward prediction.

\subsection{Distinguishing forward and backward prediction}

We draw a distinction between two forms of prediction of~$Y$ from $X$:  
\emph{forward} prediction models the mechanism by which $X$ influences $Y$;
\emph{backward} prediction forecasts $Y$ from $X$ indirectly, by exploiting correlations through the context $W$. 
Because~$W$ may be redundantly encoded within $X$, we cannot simply remove~$W$ from the features to evaluate the predictive power along the forward pathway. Instead, we define forward and backward prediction based on conditional independence statements involving $Y,h(X),$ and $W$.

\begin{definition}[Backward and forward prediction]
\label{def:backward}
A predictor~$h:\Xcal \to \Ycal$ is a (pure) backward predictor of $Y$ if $h(X)$ is conditionally independent of $Y$ given $W$. $$h(X) \bot Y \vert W$$
A predictor~$h:\Xcal \to \Ycal$ is a (pure) forward predictor of $Y$ if $h(X)$ is independent of $W$. $$h(X) \bot W$$
\end{definition}
Most classifiers will be not be pure forward or pure backward predictors, but instead $h(X)$ will have some correlation with $Y$ that goes through $W$ and some correlation that is independent of $W$.
By comparing the loss achieved by a classifier~$h$ to one of our backward baselines, we can understand how close to a backward predictor the classifier is.

\paragraph{Backward prediction as random targeting.}
Connecting Definition~\ref{def:backward} to our motivating question, we see that using a backward predictor as the basis for intervention on \emph{individuals} is fruitless.
In particular, once we condition on a category defined by $W$, backward predictions $h(X)$ can be randomized across individuals with no loss.
\begin{fact}
    Suppose $h:\Xcal \to \Ycal$ is a backward predictor.
    For any setting of the context $W = w$, consider a randomized prediction strategy, where $R_w$ is an independent random variable distributed as $h(X) \vert W=w$.
    Then, the loss of $h$ is equal to that of random prediction according to $R_w$.
    \begin{gather*}
        \E[\ell(Y,h(X)) \vert W=w] = \E[\ell(Y,R_w) \vert W = w]
    \end{gather*}
\end{fact}
This fact follows immediately from the definition of backward prediction through conditional independence, but it gives a powerful conclusion.
Given a predictor that uses the backward pathway through $W$, once we condition on a particular setting of $W=w$, then the predictions $h(X)$ may as well be randomly assigned.
That is, using a backward predictor as the basis of intervention, is analogous to stereotyping according to the categories defined by $W$, and then targeting the intervention randomly within categories.




\section{Properties of backward baselines}
\label{sec:theory}
In this section, we develop basic theory for backward baselines, demonstrating how these baselines give us a lens into understanding backward and forward prediction.
We study the basic properties of backward baselines, establish interpretations of these baselines in terms of familiar statistical quantities, and draw connections to concepts from the study of fairness in prediction. We defer proofs of all formal claims to Appendix~\ref{app:proofs}.

\subsection{Basic properties}

Here, we establish some basic properties about backward baselines.
These properties are intuitive, but also reveal subtleties in what we can(not) conclude about backward and forward prediction from backward baselines.
We start with three simple properties of backward baselines, that help us to compare the predictive power from $X$ to the predictive power from $W$.
\begin{proposition}
\label{prop:basic}
The following properties of backward baselines hold.

\emph{(a)}~~ When $X$ encodes $W$, there exists a predictor $h^*:\Xcal \to \Ycal$ that achieves loss at most the backward prediction baseline.
\begin{gather*}
    \ell_\Dist(Y,h^*(X)) \le \ell_\Dist(Y,\gs(W))
\end{gather*}
\emph{(b)}~~If $h:\Xcal \to \Ycal$ is a backward predictor, then its loss is at least the backward baselines.
\begin{align*}
    \ell_{\Dist}(Y,\gs(W)) \le \ell_\Dist(Y,h(X))\quad\text{and}\quad
    \ell_\Dist(h(X),\gh(W)) \le \ell_\Dist(Y,h(X))
\end{align*}
\emph{(c)}~~ If $h:\Xcal \to \Ycal$ is a forward predictor, then $\gh$ is comparable to a constant predictor. Formally,
\begin{align*}
    \ell_\Dist(h(X),\gh(W)) \ge \argmin_{\yhat \in \Ycal}~ \E[\ell(h(X),\yhat)]\quad\text{and}\quad
    \ell_\Dist(Y,\gh(W)) \ge \argmin_{\yhat \in \Ycal}~ \E[\ell(Y,\yhat)]\,.
\end{align*}
\end{proposition}

These straightforward properties provide a foundation for reasoning about backward and forward prediction.
Proposition~\ref{prop:basic}(a) establishes that the backward prediction baseline is reasonable minimum standard for predictive accuracy from $X$.
Proposition~\ref{prop:basic}(b)-(c) can be viewed as one-sided tests that let us demonstrate that a predictor is not a (pure) backward or forward predictor.

For a backward predictor, the backward baselines lower bound the loss $\ell_\Dist(Y,h(X))$.
On the other hand, for a forward predictor that achieves nontrivial loss (i.e., beating a constant), the backward baselines upper bound the loss.
While Proposition~\ref{prop:basic}(b)-(c) each provide one-sided tests, together they can tell a rich story.
For instance, suppose a forward predcitor $f$ and backward predictor $b$ achieve similar loss $\ell_\Dist(Y,f(X)) \approx \ell_\Dist(Y,b(X))$.
We may distinguish these cases by backward rounding the predictors to $g^f$ and $g^b$.
Rounding $f$ to $g^f$ will cause a significant deterioration in loss (to that of a constant predictor), but the rounded backward predictor $g^{b}$ will maintain the predictive power of $b$.
In this case, we may still decide to reject $f$ if it achieves mediocre accuracy, but cannot reliably reject it on the basis of being a backward predictor.


\subsection{Rounding recovers optimal backward prediction}
As discussed, we can define backward baselines in terms of the optimal predictor $\gs$ of $Y$ from $W$, and also in terms of the backward-rounded predictor $\gh$ of $h(X)$ from $W$.
In generality, these two predictors realize different baselines; however, if $h(X)$ is an accurate predictor of $Y$, then intuitively, it would seem that the baselines over $\gs$ and $\gh$ might be similar.
For instance, for classification according to the zero-one loss and regression according to the squared loss, these predictors have closed forms.
\begin{align*}
\begin{array}{lcc}
\textbf{Zero-one}& \gs(w) = \argmax_{\yhat \in \Ycal} \Pr[Y = \yhat \vert W = w]\quad & \gh(w) = \argmax_{\yhat \in \Ycal} \Pr[h(X) = \yhat \vert W = w] \\
& ~ & \\
\textbf{Squared} &\gs(w) = \E[Y \vert W = w] & \gh(w) = \E[h(X) \vert W = w]
\end{array}
\end{align*}
We introduce the following technical conditions, which are useful for analyzing various properties of backward baselines.
\begin{definition}[Confidence]
\label{def:confident}
A classifier $h:\Xcal \to \Ycal$ is (over)confident on $Y$ over $W$ if
\begin{gather*}
    \Pr[h(X) = \gs(W)] \ge \Pr[Y = \gs(W)].
\end{gather*}
\end{definition}
Intuitively, confidence says that $h(X)$ does not underestimate the probability that $Y$ takes it's most likely value within the context $W$.
Such (over)confidence of classifiers is typically observed in practice \cite{guo2017calibration}.
\begin{definition}[Weak calibration]
\label{def:calibration}
A predictor $h:\Xcal \to [0,1]$ is weakly calibrated\footnote{This notion of weak calibration was introduced recently by \cite{gopalan2022lowdegree}, who refer to it as degree-$2$ calibration.} to $Y$ over $W$ if
\begin{align*}
    \E[Y \vert W] = \E[h(X) \vert W]\qquad\text{and}\qquad
    \E[Yh(X) \vert W] = \E[h(X)^2 \vert W].
\end{align*}
\end{definition}
Weak calibration rules out predictors that blatantly ignore variation in $Y$ based on the context $W$ (including pure forward predictors).
Definition~\ref{def:calibration} relaxes traditional notions of calibration~\cite{dawid} and is implied by loss minimization, both in theory and our experiments. We show that under these conditions, backward rounding obtains optimal prediction of $Y$ from $W$.
\begin{proposition}[Informal]
\label{prop:gh-equals-gs}
For a confident classifier $h:\Xcal \to \set{0,1}$ or a weakly calibrated predictor~$h:\Xcal \to [0,1]$, we have $\gh = \gs$ for the zero-one loss and squared loss, respectively.
\end{proposition}
The interchangeability of $\gs$ and $\gh$ may be useful practically and conceptually.
For instance, the analysis of Proposition~\ref{prop:gh-equals-gs} reveals that the backward rounding baseline lower bounds the backward prediction baseline, $\ell_\Dist(h(X),\gh(W)) \le \ell_\Dist(Y,\gs(W))$ (which, in turn, gives a strengthening of Proposition~\ref{prop:basic}(b) under confidence or weak calibration).

\subsection{Measuring forward predictive power}
A key motivation for our study of backward baselines was the observation that, given a predictor~$h$, determining the extent of forward prediction may be challenging.
We show that under natural conditions, the backward rounding baseline for $\gh$ reveals insight into the forward predictive power of $h$.
Conveniently, evaluating this baseline only requires black-box access to the predictive model and $(X,W)$ samples---not labels $Y$.
The lightweight nature of the baseline makes it an appealing option to audit for backward prediction, especially for proprietary predictive models.
Concretely, we show that the backward rounding baseline gives insight into the covariance between $h(X)$ and $Y$ after conditioning on $W$.
\begin{proposition}
\label{prop:cov}
Suppose a classifier $h:\Xcal \to \set{0,1}$ is confident on $Y$ over $W$.
Let $\ell_W(h,\gh)$ denote the backward rounding baseline $\Pr[h(X) \neq \gh(W) \vert W]$ conditioned on $W$.  Then,
\begin{gather*}
    \Cov(h(X),Y \vert W) \le \Var(h(X) \vert W) = \ell_W(h,\gh) \cdot (1-\ell_W(h,\gh)) \le \Var(Y \vert W)\,.
\end{gather*}
If a predictor $h:\Xcal \to [0,1]$ is weakly calibrated to $Y$ over $W$, then
\begin{gather*}
    \E[(h(X) - \gh(W))^2] = \E[(Y - \gs(W))^2] - \E[(Y - h(X))^2] = \E_W[\Cov(Y,h(X) \vert W)]\,.
\end{gather*}
\end{proposition}
In other words, if $h(X)$ carries lots of information about $Y$, even after conditioning on $W$, then the backward rounding baseline will be large.
The arguments to establish Proposition~\ref{prop:cov} are elementary, but the consequences are powerful.
An auditor, who is given only black-box access to a classifier or predictor $h$, can reliably determine when $h$ is a backward predictor by evaluating the backward rounding baseline without any labels $Y$ from the true distribution.
Concretely, the backward rounding baseline allows the auditor to establish an upper bound on the amount of information about $Y$ contained in $h(X)$ that isn't explained by $W$.

In the classification setting, the bound obtained by the rounding baseline is an inequality, but is tighter than the bound given by the backward prediction baseline.
In the regression setting, the rounding baseline also characterizes the difference between the backward prediction baseline and the expected loss of $h$, which would otherwise require labeled outcomes $Y$ to evaluate.
In Appendix~\ref{app:proofs}, we describe an additional backward baseline for classification, which use labels from $Y$ to gives an exact characterization of the forward predictive power of $h$.

\subsection{Backward baselines and demographic parity}
When $W$ is defined by demographic features that are considered to be sensitive attributes, forward prediction recovers the notion of \emph{demographic parity} from the literature on fair machine learning \cite{fta}.
While a natural desideratum for equal treatment under a decision rule, the shortcomings of demographic parity as a notion of fairness have been documented extensively \cite{fta,liu2018delayed}.
As such, requiring pure forward prediction may result in unintended and undesirable consequences, just as blinding predictors of a sensitive attribute can.

Exploring the analogy between backward baselines and fair prediction sheds new light on demographic parity and stereotyping.
In Appendix~\ref{app:fairness}, we formalize a duality between forward and backward prediction.
Translating the duality into the language of fairness, the optimal unconstrained prediction decomposes into the optimal prediction under demographic parity plus the optimal ``stereotyping'' prediction that makes its judgments solely based on the sensitive attribute.

\section{Empirical evaluation of backward baselines}
\label{sec:experiments}
The goal of our experiments is to empirically evaluate backward baselines. Toward this goal, we searched for datasets that meet at least four important criteria: 
\begin{enumerate}
\item The outcome variable demonstrably lies in the future relative to the features. 
\item The dataset contains general demographic background variables, as well as features specific to the prediction task. 
\item Non-trivial prediction accuracy is possible. 
\item Individual-level microdata are publicly available.
\end{enumerate}
Many machine learning datasets are unclear about the temporality of the outcome variable, thus falling short of the first criterion. For example, several datasets about credit default prediction do not clarify whether data points correspond to individuals who have already defaulted, or individuals that ended up defaulting some specific time after feature collection.

Well-suited to our evaluation are longitudinal panel surveys. Each panel consists of some number of survey participants who are interviewed in multiple rounds (or waves). By taking features from one round to predict outcomes in a later round, we can create prediction problems where outcomes and features are temporally well-separated. We choose two major panel surveys relating to medical expenditure and income. Complementing these panel surveys, we also consider a notorious dataset from the criminal legal domain. Extended results and full details are in Appendix~\ref{app:experiments} and Appendix~\ref{sec:code}. The code is available at:
\begin{center}
    \url{https://github.com/socialfoundations/backward_baselines}
\end{center}

\subsection{Medical Expenditure Panel Survey (MEPS)}

The Medical Expenditure Panel Survey (MEPS) is a set of large-scale surveys of families and individuals, their medical providers, and employers across the United States, aimed at providing insights into health care utilization. We work with the publicly available MEPS 2019 Full Year Consolidated Data File. The dataset we consider has 28,512 instances corresponding to all persons that were part of one of the two MEPS panels overlapping with calendar year 2019. Specifically, Panel 23 has Rounds 3--5 in 2019, and Panel 24 has rounds 1--3 in 2019. Round 3 of Panel 23 and Round 1 of Panel 24 are the first of each panel in 2019. The survey distinguishes between demographic variables and variables corresponding to survey questions in of the rounds of the two panels. We create a prediction task whose goal is to predict a full-year outcome from Round 3 of Panel 23 and Round 1 of Panel 24. The target variable measures the total health care utilization across the year. We create a roughly balanced binarization of the target variable. A precise definition and further details are in the appendix.

We compute backward baselines in terms of the features \emph{age}, \emph{race}, \emph{age} and \emph{race} together, as well as all variables designated as \emph{demographic} by the survey documentation. These include additional variables relating to age, race and ethnicity, marital status, nationality, and languages spoken. Figure~\ref{fig:meps-zeroone} summarizes our findings. In particular, backward baselines trained on all demographic background variables match nearly all of the predictive performance of the classifiers trained on all features, similarly across three different prediction models.
An extended set of figures is included in Appendix~\ref{app:experiments}.

\begin{figure}[h]
\includegraphics[width=0.98\linewidth]{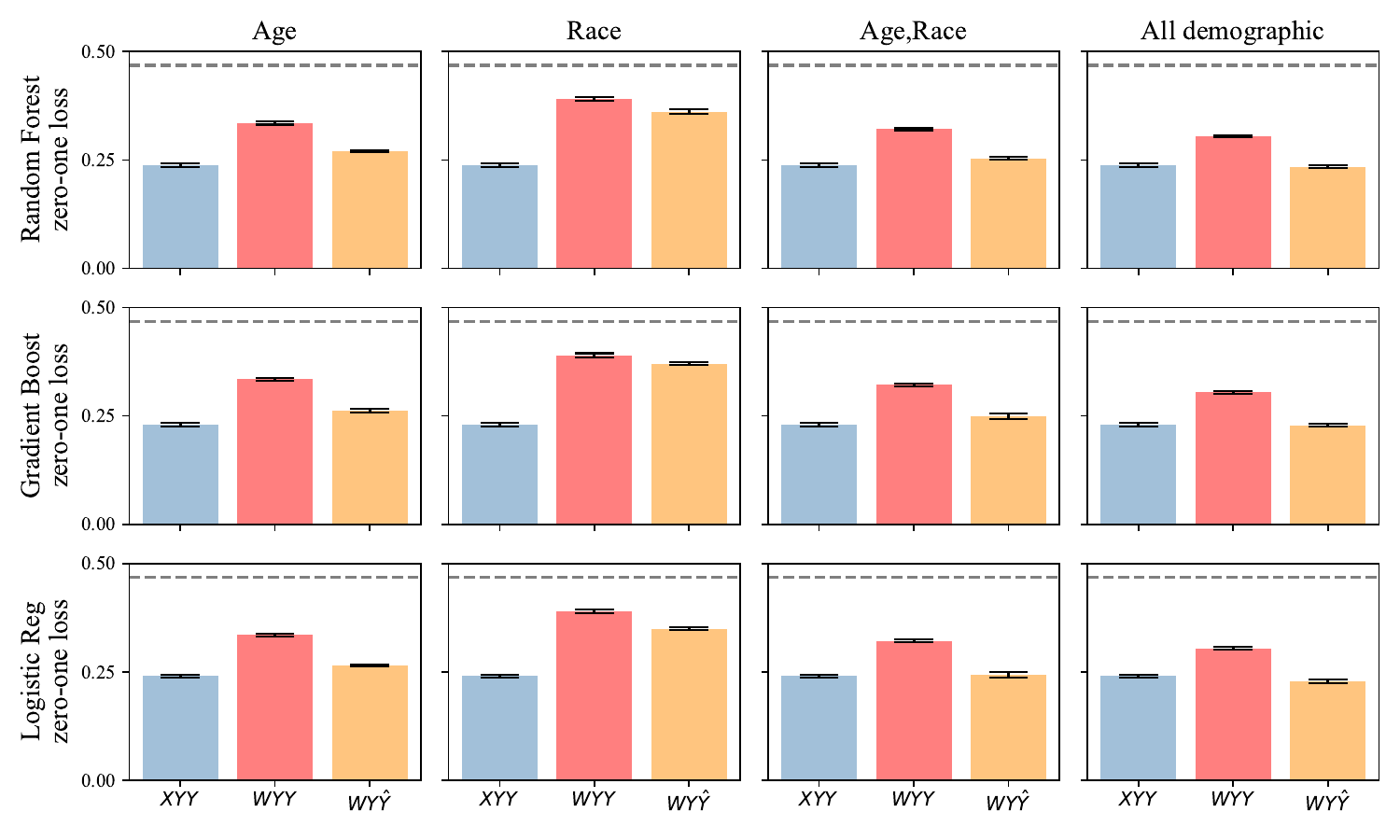}
\includegraphics[width=0.98\linewidth]{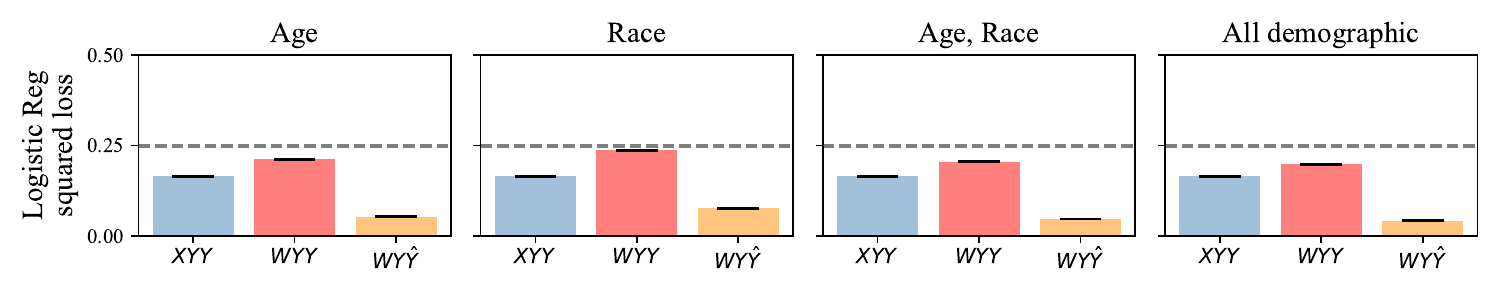}
\includegraphics[width=0.98\linewidth]{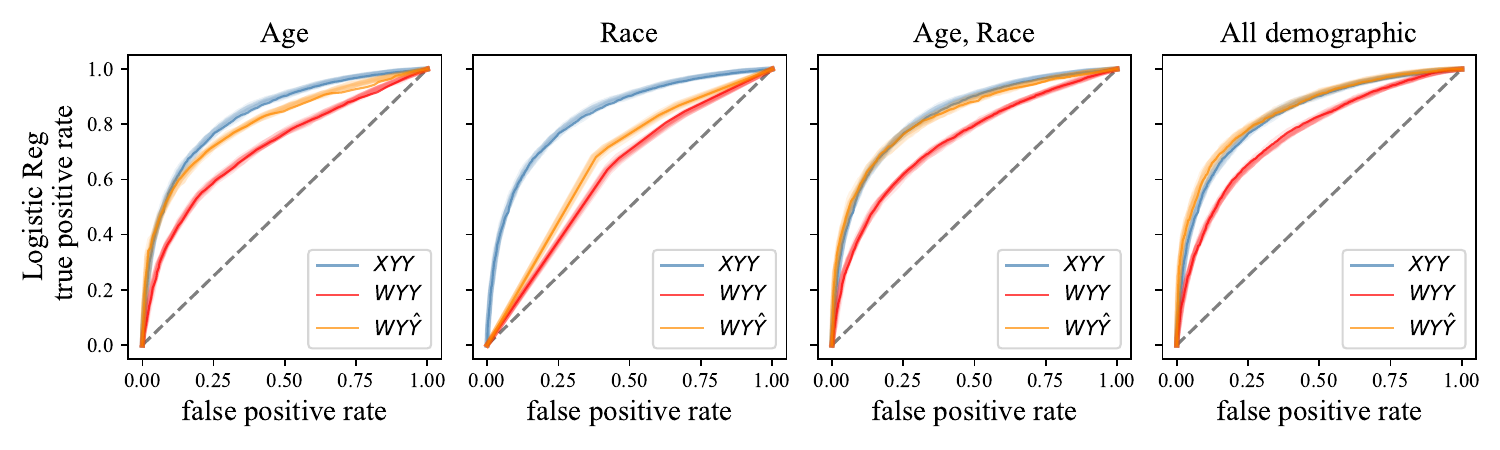}
\caption{Backward baselines on MEPS, columns are different features, rows are different classifiers (random forest, gradient boosting, logistic regression) and metrics (zero-one loss, squared loss, ROC curves). Label $XYY$ denotes standard training and testing, label $WYY$ is the backward prediction baseline, label $WY\hat Y$ is the backward rounding baseline. Gray dashed line indicates performance of constant predictor. Error bars represent a standard deviation across $10$ random seeds.}
\label{fig:meps-zeroone}
\end{figure}

\subsection{Survey of Income and Program Participation (SIPP)}

The Survey of Income and Program Participation (SIPP) is an import longitudinal survey conduced by the U.S.~Census Bureau, aimed at capturing income dynamics as well as participation in government programs.

We consider Wave 1 and Wave 2 of the SIPP 2014 panel data. The target variable is based on the official poverty measure (OPM), a cash-income based measure of poverty. We compute this measure based on Wave 2 data. We again discretize the measure to obtain to roughly balanced classes for our binary prediction task. The goal is to predict this outcome based on features collected in Wave 1. After cleaning and preprocessing our data contains 39720 rows and 54 columns. We consider background variables \emph{education}, \emph{race}, \emph{education} and \emph{race} together, as well as all demographic variables, specifically, \emph{age}, \emph{gender}, \emph{race}, \emph{education}, \emph{marital status}, \emph{citizenship status}. In Figure~\ref{fig:sipp-main}, we restrict our attention to the logistic regression model. The other models perform similarly and the full set of results can be found in Appendix~\ref{app:experiments}.

\begin{figure}[h]
\includegraphics[width=0.98\linewidth]{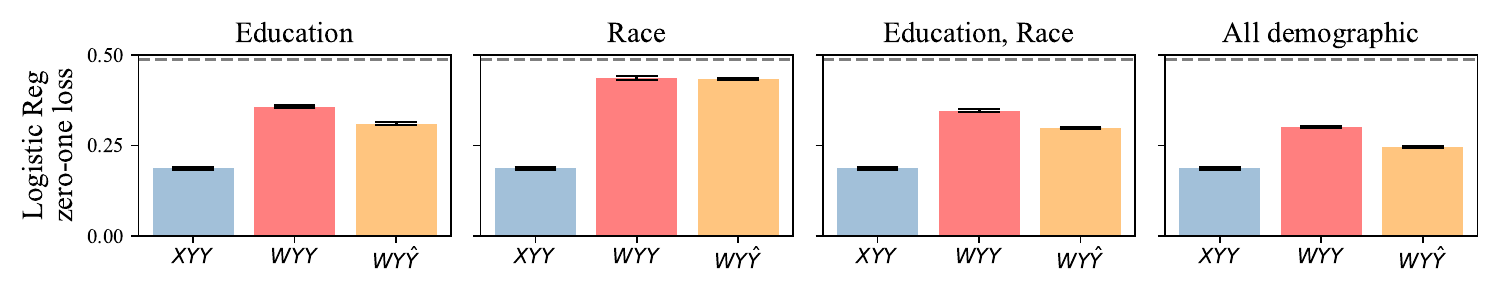}
\includegraphics[width=0.98\linewidth]{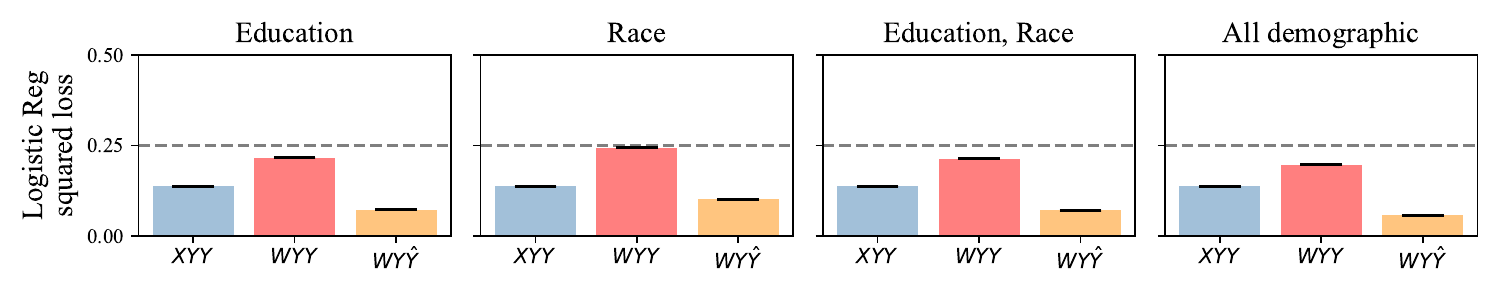}
\includegraphics[width=0.98\linewidth]{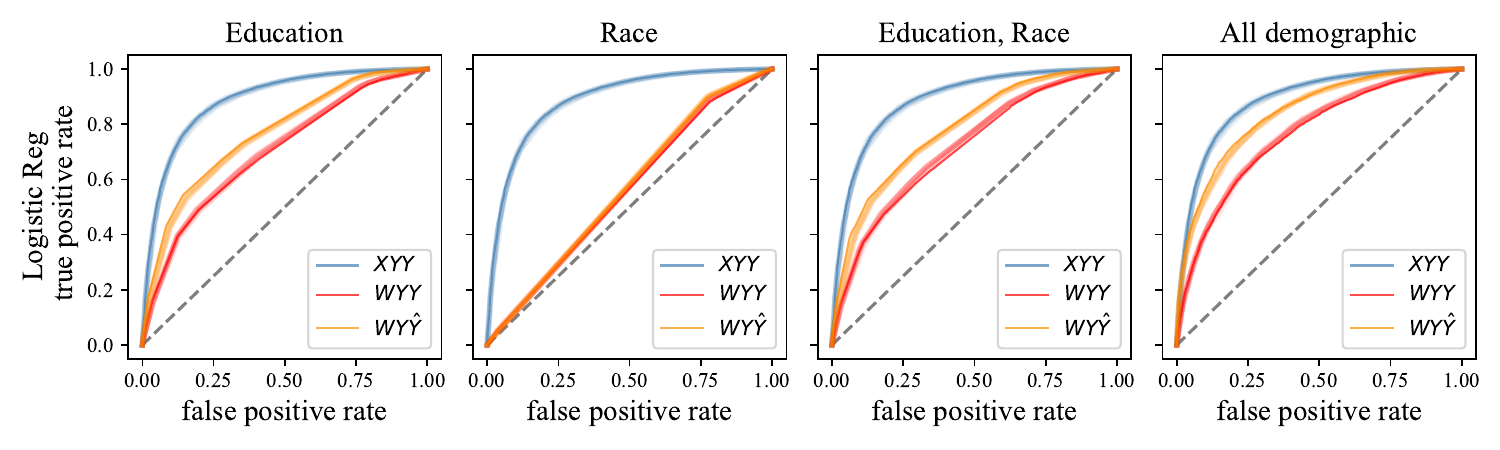}
\caption{Backward baselines on SIPP.}
\label{fig:sipp-main}
\end{figure}

\subsection{ProPublica COMPAS Recidivism Scores}

A proprietary recidivism risk score, called COMPAS, was the subject of a  notorious investigation into racial bias by ProPublica~\cite{angwin2016machine} in 2016. As part of the investigation ProPublica released a dataset of COMPAS scores about defendants associated with two-year recidivism outcomes.
The dataset released by ProPublica has significant and well-documented issues that make it inadequate for the development of new risk scores as well as fairness interventions~\cite{bao2021its, barenstein2019propublica}.
In experimenting with the COMPAS data set, our primary goal is to demonstrate the effectiveness of backward baselines in auditing problematic risk predictors.
The results of backward baselines echo earlier findings that the performance COMPAS scores can be achieved by simple models~\cite{rudin2020age, wang2022pursuit}.

Note that we do not have access to the training data used for producing the COMPAS scores as is common in algorithmic audit scenarios. This is, fortunately, not required for evaluating backward baselines. We only need the scores, as well as associated demographic information. Figure~\ref{fig:compas-main} evaluates backward baselines against the COMPAS scores. The results are rather striking in how well backward baselines do in comparison. In particular, a single feature (prior convictions) appears to account for all of the predictive power of the COMPAS score.

\begin{figure}[h]
\includegraphics[width=0.98\linewidth]{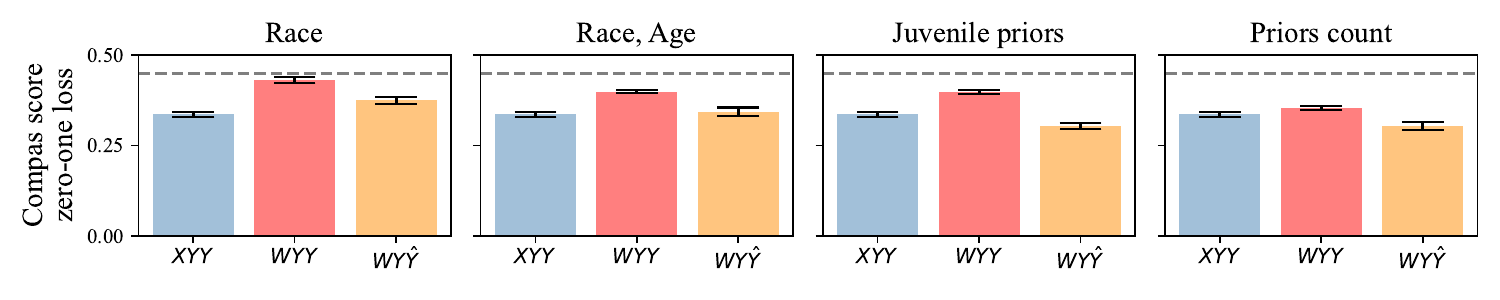}
\includegraphics[width=0.98\linewidth]{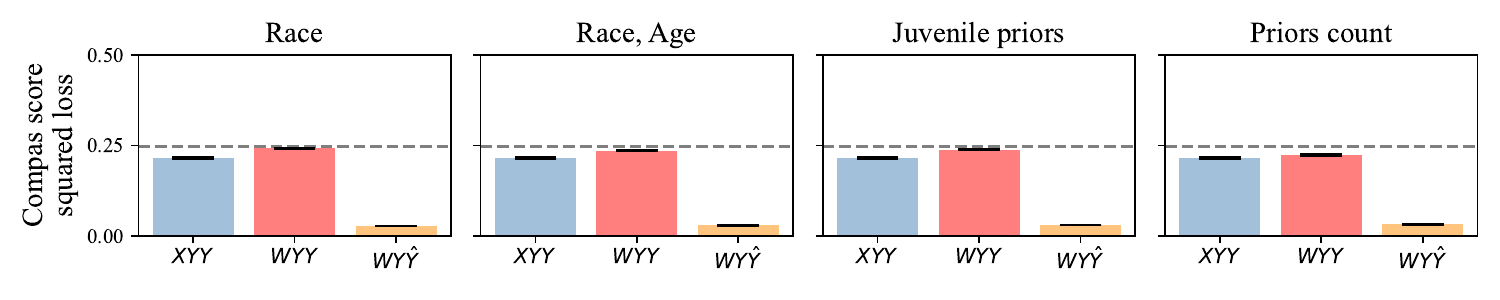}
\includegraphics[width=0.98\linewidth]{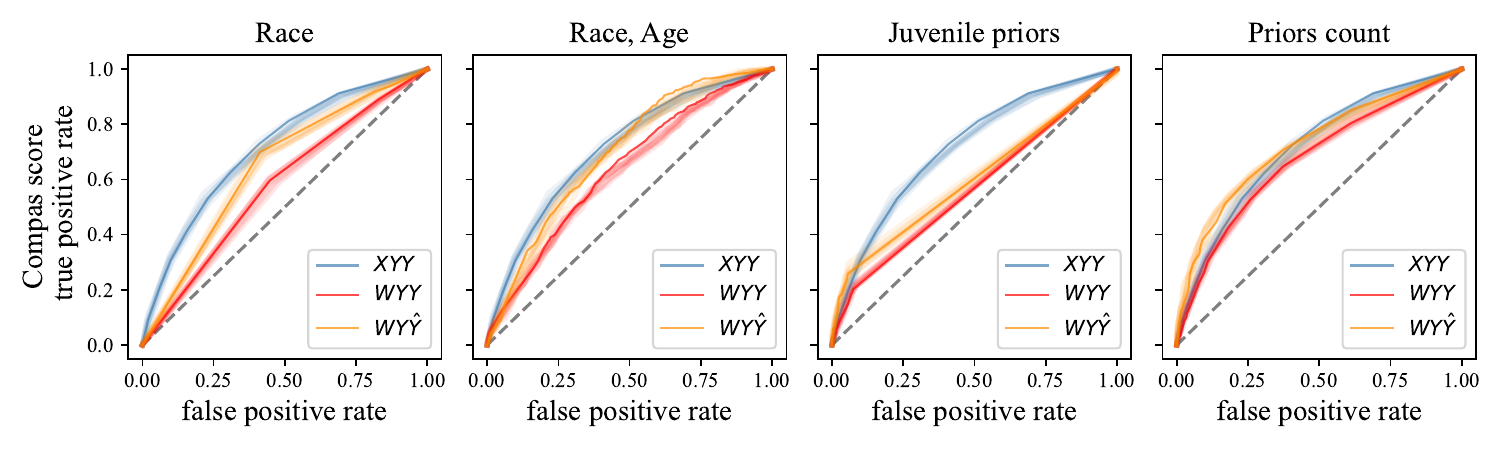}
\caption{Backward baselines on COMPAS}
\label{fig:compas-main}
\end{figure}

\clearpage
\section{Discussion}
\label{sec:discussion}
It might seem that the question we ask runs headlong into the centuries-old problem of induction: \emph{How do we  draw conclusions about the future based on past experience?}
But our study addresses a distinct and more specific question:
\emph{Does a given predictor capitalize on individual behaviors that influence the outcome or on historical patterns that correlate with the outcome?}
Backward baselines ask about induction \emph{from what}.
Initiating this investigation, we start from factors that apparently predate the point of prediction, such as demographic background variables, and test to what extent a predictor utilizes these factors.
Our findings---that backward prediction often serves a significant role in forecasting individuals' outcomes---adds relevant evidence
for the ongoing deliberation of the meaning of individual risk scores \cite{sandroni2003reproducible,dawid2017individual,dwork2021outcome,dwork2021pseudo}.

Our present investigation into backward baselines is limited to settings where the variables defining a past context $W$ are measured and observed by the auditor.
A fundamental question in evaluating backward baselines is which variables constitute the right choice for context $W$. We emphasized the role of time in deciding what is outside the individual's control.
Some factors are obviously in the past, e.g.,\ place of birth, and parents' educational attainment. Other factors, such as race, gender, and individual's educational attainment, involve the individual at present but are nonetheless socially constituted. Time alone is therefore an imperfect guiding principle in choosing what we count as a suitable background variable~$W$. Choosing $W$ appropriately is not a purely technical question, but rather is up for debate based on the context and scope of the prediction task.

\subsection{Additional Related Works}
On the level of techniques, backward baselines bear resemblance to a number of tools developed in the causal inference and machine learning communities.
Backward baselines do not make any assumptions on the underlying causal structure between $X$, $Y$, and $W$.
Still, the backward baseline toolkit is similar in ways to tools developed in settings for understanding the causal structure between variables, both for measuring confounding \cite{mcnamee2003confounding,pearl2009causality} and mediation analysis \cite{mackinnon2007mediation}.
At first glance, backward baselines may also feel similar to the study of spurious correlation, which has received considerable attention in the ML literature (e.g., \cite{sagawa2020investigation,srivastava2020robustness,veitch2021counterfactual}).
We caution, however, that correlation with background (or individual) features should not be understood as ``spurious''.
Instead, correlations with background features reveal important structure in the data distribution, how the predictor exploits this structure, and in turn, when intervening on the basis of prediction may be ineffective.

As discussed, backward baselines also add new color to concepts studied in algorithmic fairness.
In particular, under a specific causal interpretation, forward prediction may be understood similarly as the notion of Counterfactual Fairness \cite{kusner2017counterfactual}.
This connection can be made somewhat formal, as the latter notion has been shown to be closely related to the notion of Demographic Parity as well \cite{rosenblatt2022counterfactual}.
Indeed, in some accounts of fairness, understanding the causal pathways of prediction is essential \cite{kilbertus2017avoiding}.

Finally, the findings of a recent empirical study of prediction systems in the American public school system are closely related to our theoretical work.
\cite{perdomo2023difficult} studies the Early Warning System (EWS) of the Wisconsin Department of Public Instruction.
The research concludes that, largely, the individual-risk prediction sytems is (a) effective at prediction, and (b) reliant predominantly on environmental features (e.g., what percent of an individual's school qualifies for free or reduced lunch).
In this way, while the predictions are accurate, they are backward predictors, and do not provide an effective tool for intervention on individuals' educational plans.

\subsection{Conclusions}
Our contribution has a normative, a theoretical, and an empirical component.
We argue that the distinction between predicting the future of an individual and reproducing the past is central to the debate around where and how we should use statistical methods to make consequential decisions.
The effectiveness of backward prediction, when observed, should question support for prediction as policy, and instead redirect focus toward interventions that target the background conditions.

Theoretically, we begin to develop a statistical learning theory of backward baselines.
The theory helps simplify the landscape of possible backward baselines, while clarifying how to interpret different backward baselines.
A notable outcome of our theory is that it supports the use and interpretation of a backward baseline that requires no observed outcomes.
At the outset, it was not obvious that a meaningful backward baseline without measurement of the target variable is possible.
This finding enables \emph{auditing without measured outcomes}: An investigator can probe a predictive system with access to only background variables and predictions.

On the empirical side, we show the strength and versatility of backward baselines on a variety of datasets. Utilizing multiple waves of longitudinal panel surveys, our evaluation is careful about the temporality of features and outcomes. Along the way, we contribute to a better empirical understanding of how machine learning leverages past contexts to predict future life outcomes.
In conclusion, we propose backward baselines as a simple, broadly applicable tool to strengthen evaluation and audit practices in the use of machine learning.

\section*{Acknowledgments}
We thank Rediet Abebe for insightful and formative interactions throughout the course of this work. We thank Juan C.~Perdomo for helpful discussions and feedback. We thank Ricardo Sandoval for providing us with code for the SIPP data and the associated prediction task. \textbf{MPK} is supported by the Miller Institute for Research in Basic Science and the Simons Collaboration on the Theory of Algorithmic Fairness. Authors listed alphabetically.

\clearpage
\bibliographystyle{alpha}
\bibliography{main}

\clearpage

\appendix

\section{Omitted Proofs}
\label{app:proofs}

\begin{proposition*}[Restatement of Proposition~\ref{prop:basic}]
The following properties of backward baselines hold.

\emph{(a)}~~ When $X$ encodes $W$, there exists a predictor $h^*:\Xcal \to \Ycal$ that achieves loss at most the backward prediction baseline.
\begin{gather*}
    \ell_\Dist(Y,h^*(X)) \le \ell_\Dist(Y,\gs(W))
\end{gather*}
\emph{(b)}~~If $h:\Xcal \to \Ycal$ is a backward predictor, then its loss is at least the backward baselines.
\begin{align*}
    \ell_{\Dist}(Y,\gs(W)) \le \ell_\Dist(Y,h(X))\quad\text{and}\quad
    \ell_\Dist(h(X),\gh(W)) \le \ell_\Dist(Y,h(X))
\end{align*}
\emph{(c)}~~ If $h:\Xcal \to \Ycal$ is a forward predictor, then $\gh$ is comparable to a constant predictor. Formally,
\begin{align*}
    \ell_\Dist(h(X),\gh(W)) \ge \argmin_{\yhat \in \Ycal}~ \E[\ell(h(X),\yhat)]\quad\text{and}\quad
    \ell_\Dist(Y,\gh(W)) \ge \argmin_{\yhat \in \Ycal}~ \E[\ell(Y,\yhat)]\,.
\end{align*}
\end{proposition*}
\begin{proof}[Proof of Proposition~\ref{prop:basic}]
We prove each statement separately.

\emph{(a)}~~By the assumption that $X$ encodes $W$, i.e.,\ that $I(W;X) = H(W)$, there exists a computable map $M:\Xcal \to \Wcal$ such that for any $(X,W,Y) \sim \Dist$, $M(X) = W$.
Thus, the predictor $h^*:\Xcal \to \Ycal$ defined as the composition of $\gs$ and $M$,
\begin{gather*}
    h^*(X) = \gs \circ M(X) = \gs(W)
\end{gather*}
is feasible, and achieves loss $\ell_\Dist(Y,h^*(X)) = \ell_\Dist(Y,\gs(W))$.

\emph{(b)}~~Suppose $h$ is a backward predictor; that is, $h(X) \bot Y \vert W$.
Consider the loss achieved by $h$ on $\Dist$.
\begin{align*}
\ell_\Dist(Y,h(X)) &= \E[\ell(Y,h(X))]\\
&= \E_W\E[\ell(Y,h(X)) \vert W]
\end{align*}
Note that by the conditional independence of $h(X)$ and $Y$, we can take the expectation over $X$ and $Y$ conditioned on $W$ separately.
Then, the expected loss over the choice of $h(X) \in \Ycal$ is lower bounded by the optimal choice $\yhat \in \Ycal$.
\begin{gather*}
    \E_{h(X)\vert W}\E_{Y \vert W}[\ell(Y,h(X)) \vert W] \ge \min_{\yhat \in \Ycal}~ \E_{Y \vert W}[\ell(Y,\yhat) \vert W] = \E_{Y \vert W}[\ell(Y,\gs(W)) \vert W]
\end{gather*}
Thus, in all, we conclude that $\ell_\Dist(Y,\gs(W)) \le \ell_\Dist(Y,h(X))$.
The second inequality follows similarly, by lower bounding the expected loss over the draw of $Y$ by the optimal $\hat{y} \in \Ycal$, which results in $\gh(W)$.

\emph{(c)}~~Suppose $h$ is a forward predictor; that is $h(X) \bot W$.
Consider the definition of $\gh$,
\begin{align*}
    \gh(W) &= \argmin_{\yhat \in \Ycal}~\E[\ell(h(X),\yhat) \vert W]\\
    &= \argmin_{\yhat \in \Ycal}~\E[\ell(h(X),\yhat)]
\end{align*}
where the equality between the conditional and unconditional expectation follows by independence.
Thus, $\gh:\Wcal \to \Ycal$ must be a constant predictor, and can only hope to compete with the best fixed prediction in predicting $Y$ or $h(X)$.
\end{proof}


\begin{proposition*}[Formal restatement of Proposition~\ref{prop:gh-equals-gs}]
Suppose a classifier $h:\Xcal \to \set{0,1}$ is confident on $Y$ over $W$.
Then,
\begin{gather*}
\gs(W) = \argmax_{\yhat \in \Ycal} \Pr[Y = \yhat \vert W] = \argmax_{\yhat \in \Ycal} \Pr[h(X) = \yhat \vert W] = \gh(W).
\end{gather*}
Suppose a predictor $h:\Xcal \to [0,1]$ is weakly calibrated to $Y$ over $W$.
Then,
\begin{gather*}
\gs(w) = \E[Y \vert W] = \E[h(X) \vert W] = \gh(W).
\end{gather*}
\end{proposition*}
\begin{proof}[Proof of Proposition~\ref{prop:gh-equals-gs}]
First, we prove the equality for classifiers.
Note that minimization is entirely determined by which side of $1/2$ the probability that the outcome is $1$.
That is, $\argmin_{\yhat \in \Ycal}~ \Pr[Y \neq \yhat \vert W]$ returns the indicator of whether $\Pr[Y = 1 \vert W] > 1/2$.
Thus, as long as $\Pr[Y =1 \vert W]$ and $\Pr[h(X) = 1 \vert W]$ are on the same side of $1/2$, then the equality follows.
By confidence, if $\gs(W) = 1$, then $1/2 < \Pr[Y = 1 \vert W] \le \Pr[h(X) = 1 \vert W]$, so $\gh(W) = 1$ as well.
The statement holds analogously for the case $\gs(W) = 0$.

Next, we prove the equality for regression predictors.
By the definition of weak calibration, we have that $h$ matches expectations with $Y$ conditional on $W$.
\begin{gather*}
    \E[h(X) \vert W] = \E[Y \vert W].
\end{gather*}
Thus, by the closed-form solution for $\gs$ and $\gh$, we have the stated equality.
\begin{align*}
\gs(W) &= \E[Y \vert W]\\
&= \E[h(X) \vert W]\\
&= \gh(W)
\end{align*}
\end{proof}

\begin{proposition*}[Restatement of Proposition~\ref{prop:cov}]
Suppose a classifier $h:\Xcal \to \set{0,1}$ is confident on $Y$ over $W$.
Let $\ell_W(h,\gh)$ denote the backward rounding baseline $\Pr[h(X) \neq \gh(W) \vert W]$ conditioned on $W$.  Then,
\begin{gather*}
    \Cov(h(X),Y \vert W) \le \Var(h(X) \vert W) = \ell_W(h,\gh) \cdot (1-\ell_W(h,\gh)) \le \Var(Y \vert W)
\end{gather*}
If a predictor $h:\Xcal \to [0,1]$ is weakly calibrated to $Y$ over $W$, then
\begin{gather*}
    \E[(h(X) - \gh(W))^2] = \E[(Y - \gs(W))^2] - \E[(Y - h(X))^2] = \E_W[\Cov(Y,h(X) \vert W)].
\end{gather*}
\end{proposition*}

\begin{proof}[Proof of Proposition~\ref{prop:cov}]
Suppose $h:\Xcal \to \set{0,1}$ is confident on $Y$ over $W$.
First, the covariance $\Cov(h(X),Y \vert W)$ is upper bounded by the variance $\Var(h(X) \vert W)$.
Then, we express the variance of this Bernoulli random variable in terms of the backward rounding baseline.
Specifically, for either $\yhat \in \set{0,1}$:
\begin{align*}
\Var(h(X) \vert W) &= \Pr[h(X) = \yhat \vert W] \cdot \Pr[h(X) \neq \yhat \vert W]\\
&= \Pr[h(X) \neq \gh(W) \vert W] \cdot (1-\Pr[h(X) \neq \gh(W) \vert W]).
\end{align*}
Further, by confidence and Proposition~\ref{prop:gh-equals-gs}, we can bound the probabilities.
\begin{align*}
\Pr[h(X) \neq \gh(W) \vert W] &\le \Pr[Y \neq \gh(W) \vert W]\\
&= \Pr[Y \neq \gs(W) \vert W]\\
&\le 1/2
\end{align*}
By properties of variance of Bernoullis, $h(X)$ given $W$ is more peaked than $Y$ given $W$, so will have lower variance.

Given a weakly calibrated predictor $h:\Xcal \to [0,1]$, we expand the difference in squared loss as follows.
\begin{align*}
\E[(Y - \gh(W))^2] - \E[(Y - h(X))^2]
&= 2 \E[Y (h(X) - \gh(W))] - \E[h(X)^2 - \gh(W)^2]
\end{align*}
By the fact that $\gh(W) = \E[h(X) \vert W]$, the second term can be rewritten as the squared error between $\gh$ and $h$.
\begin{align*}
\E[h(X)^2 - \gh(W)^2] &= \E_W\E[h(X)^2 - \gh(W)^2 \vert W]\\
&= \E_W[\E[h(X)^2 \vert W] - \gh(W)^2]\\
&= \E_W[\E[h(X)^2 \vert W] - 2\gh(W)\E[h(X) \vert W] + \gh(W)^2]\\
&= \E[(h(X) - \gh(W))^2]
\end{align*}
The first term can be rewritten as the expected covariance between $Y$ and $h(X)$ conditioned on $W$.
\begin{align*}
\E[Y ((h(X) - \gh(W))]
&= \E_W\E[Y ((h(X) - \gh(W)) \vert W]\\
&= \E_W[\E[Y h(X) \vert W] - \E[Y \vert W] \gh(W)]\\
&= \E_W[\E[Y h(X) \vert W] - \E[Y \vert W] \E[h(X) \vert W]]\\
&= \E_W[\Cov(Y,h(X) \vert W)]
\end{align*}
In sum, the difference in losses is equal to
\begin{gather*}
\E[(Y-\gh(W))^2] - \E[(Y - h(X))^2] = 2\E_W[\Cov(Y,h(X) \vert W] - \E[(h(X) - \gh(W))^2]
\end{gather*}
Finally, if $h$ is weakly calibrated to $Y$ over $W$, then the expected covariance is equal to the squared of $\gh$ to $h$,
\begin{align*}
\E_W[\Cov(Y,h(X) \vert W)] &= \E_W[\E[Y h(X) \vert W] - \E[Y \vert W] \E[h(X) \vert W]]\\
&= \E_W[\E[h(X)^2 \vert W] - \gh(W)^2] \addtag \label{proof:eqn:cov:calib}\\
&= \E[(h(X) - \gh(W))^2]
\end{align*}
where (\ref{proof:eqn:cov:calib}) follows by the assumption that $h$ is weakly calibrated to $Y$ over $W$.
Thus, the difference in losses simplifies to the squared difference between $\gh$ and $h$.
\end{proof}

\paragraph{An alternative backward baseline for classification covariance.}

We present an additional backward baseline for classifiers that may be of interest when an underlying score function is not available.
In this baseline, we manipulate the distribution over outcomes, leaving the predictions fixed.
Specifically, given a sample $(X,W,Y) \sim \Dist$, we resample the outcome $\Yt \sim \Dist_{Y \vert W}$, ensuring that $h(X)$ and $\Yt$ are conditionally independent given $W$.
We show that for the zero-one loss, the difference between the backward baseline $\ell_\Dist(\Yt,h(X))$ and $\ell(Y,h(X))$ is proportional to the expected conditional covariance of $Y$ and $h(X)$ given $W$.
\begin{proposition}
\label{prop:alternate:cov}
For any classifier $h:\Xcal \to \set{0,1}$,
\begin{gather*}
    \Pr[\Yt \neq h(X)] - \Pr[Y \neq h(X)] = 2 \E_W[\Cov(Y,h(X) \vert W)]
\end{gather*}
\end{proposition}
\begin{proof}
We expand the difference in zero-one loss by exploiting the identity that $\Pr[Y \neq h(X)] = \E[Y + h(X) - 2 Y h(X)]$ for binary $Y$ and $h(X)$, and using the fact that $\E[\Yt\vert W] = \E[Y \vert W]$.
\begin{align*}
\Pr[\Yt \neq h(X)] - \Pr[Y \neq h(X)]
&= 2 \left(\E[Y h(X)] - \E[\Yt h(X)]\right)\\
&= 2 \E_W[\E[Y h(X) \vert W] - \E[\Yt h(X) \vert W]]\\
&= 2 \E_W[\E[Y h(X) \vert W] - \E[Y\vert W] \E[h(X) \vert W]] \addtag \label{proof:eqn:cov:01}\\
&= 2 \E_W[\Cov(Y,h(X)\vert W)]
\end{align*}
where (\ref{proof:eqn:cov:01}) follows by the fact that $\Yt \sim \Dist_{Y \vert W}$ is sampled conditionally independently from the distribution on $Y$ given $W$.
\end{proof}

\section{Backward Prediction and Demographic Parity}
\label{app:fairness}
While conceived from different vantages, backward baselines and fair machine learning share similarities in perspective and technical structure.
On a technical level, pure forward prediction is equivalent to demographic parity, a notion of fairness introduced by \cite{fta}.
Based on this observation, certain insights about backward baselines have an analogue in fair prediction, and vice versa.
For instance, we note that in combination, Proposition~\ref{prop:basic} and Proposition~\ref{prop:gh-equals-gs} imply that forward predictors cannot be calibrated to $Y$ over $W$.
Translating this observation into the language of fairness in prediction, we recover a specific case of the well-known results on the incompatibility of calibration and parity-based definitions of fairness in prediction~\cite{kmr,chouldechova2017fair}.

In addition to giving insight into the backward rounding baseline, Proposition~\ref{prop:cov} shows a formal sense in which forward and backward predictors are orthogonal to one another.
In particular, for weakly calibrated regression predictors $h:\Xcal \to [0,1]$,
\begin{gather*}
    \E[(Y - \gh(W))^2] = \E[(Y - h(X))^2] + \E[(h(X) - \gh(W))^2]
\end{gather*}
is a sort of Pythagorean theorem, stating that the variation in $Y$ after accounting for $\gh(W)$ can be broken into the variation in $Y$ given $h(X)$ and the variation in $h(X)$ given $\gh(W)$.

Connecting the backward baselines framework to fairness in prediction suggests a simple algorithm for learning predictors satisfying demographic parity, that relies only on unconstrained learning primitives.
First, we learn to predict $Y$ as $h(X)$; then, we learn to predict $h(X)$ as $\gh(W)$; finally, we return $f_\alpha(X,W)$ defined as
$$f_\alpha(X,W) = h(X) - \alpha \gh(W)$$
for $\alpha \in [0,1]$.
Taking $\alpha = 1$ achieves a relaxed first-order demographic parity.
Specifically, $f_1(X,W)$ has a constant expectation over all $W$.
\begin{gather*}
    \E[f_1(X,W) \vert W] = \E[h(X) - \gh(W) \vert W] = \E[h(X)\vert W] - \gh(W) = 0
\end{gather*}
In effect, $f_1(X,W)$ predicts optimally according to $X$ then removes all variation that can be accounted for through $W$.
Other choices of $\alpha$ may be interesting to interpolate between forward and backward prediction modes.

\section{Details on Empirical Evaluation}
\label{app:experiments}
In this section, we show all figures for all baselines, classifiers, and metrics that we considered. We also provide additional details on the data sources, feature engineering, and target variable creation.

In all bar plots, the height of the bar is the mean value from $10$ different random seeds and error bars indiciate a standard deviation across $10$ different random seeds. In the case of ROC curves, the plot shows $10$ curves overlaid from $10$ different random seeds. None of the experiments require significant compute resources.

Given features $X$, context $W$, a given predictor $\hat Y$, and target variable $Y$, our plots evaluate five different methods:

\begin{itemize}
    \item $XYY$ : Train on $(X, Y)$, test model on $(X, Y)$
    \item $WYY$ : Train baseline on $(W, Y)$, test baseline on $(W, Y)$ (backward prediction baseline)
    \item $W\hat YY$ : Train baseline on $(W, \hat Y)$, test baseline on $(W, Y)$ (equivalent to backward prediction baseline)
    \item $WY\hat Y$ : Train baseline on $(W, Y)$, test baseline on $(W, \hat Y)$ (equivalent to backward rounding baseline)
    \item $W\hat Y\hat Y$ : Train baseline on $(W, \hat Y)$, test baseline $(W, \hat Y)$ (backward rounding baseline)
\end{itemize}

In the main body of the paper we included only two baselines and ommitted the equivalent ones. 

\paragraph{Models.} We use three standard models available in the Python {\tt sklearn} package. We do no or only minimal hyperparameter tuning:
\begin{itemize}
    \item Gradient boosting: {\tt GradientBoostingClassifier()}
    \item Random Forests: {\tt RandomForestClassifier()}
    \item Logistic regression:\\ {\tt make\_pipeline(StandardScaler(), LogisticRegression(max\_iter=1000, tol=0.1))}
\end{itemize}
It is possible that other model families achieve better accuracy. However, on the kind of tabular data we experiment with ensemble methods such as random forests or gradient boosting tend to achieve state-of-the-art performance. We include a reference implementation of these five methods in Section~\ref{sec:code}.

\subsection{Medical Expenditure Panel Survey (MEPS)}

For extensive documentation and background on this survey, see: \url{https://www.meps.ahrq.gov/mepsweb/}

\paragraph{Data sources and use conditions.}
Our dataset is constructed from the 2019 MEPS data. The MEPS 2019 Full Year Consolidated Data File (HC-216) is available online at \url{https://meps.ahrq.gov/mepsweb/data_stats/download_data_files_detail.jsp?cboPufNumber=HC-216}. The same website contains extensive documentation regarding features and data collection. The MEPS data use agreement is available online: \url{https://meps.ahrq.gov/data_stats/download_data/pufs/h216/h216doc.shtml#DataA}.

\paragraph{Features.}
Features for Round 3 of Panel 23 and Round 1 of Panel 24 have a suffix of 31 or 31X in the data. We include all of these in the dataset, as well as all demographic features:
FCSZ1231, FCRP1231, RUSIZE31, RUCLAS31, FAMSZE31, FMRS1231, FAMS1231, REGION31, REFPRS31, RESP31, PROXY31, BEGRFM31, BEGRFY31, ENDRFM31, ENDRFY31, INSCOP31, INSC1231, ELGRND31, PSTATS31, SPOUID31, SPOUIN31, ACTDTY31, RTHLTH31, MNHLTH31, CHBRON31, ASSTIL31, ASATAK31, ASTHEP31, ASACUT31, ASMRCN31, ASPREV31, ASDALY31, ASPKFL31, ASEVFL31, ASWNFL31, IADLHP31, ADLHLP31, AIDHLP31, WLKLIM31, LFTDIF31, STPDIF31, WLKDIF31, MILDIF31, STNDIF31, BENDIF31, RCHDIF31, FNGRDF31, ACTLIM31, WRKLIM31, HSELIM31, SCHLIM31, UNABLE31, SOCLIM31, COGLIM31, VACTDY31, VAPRHT31, VACOPD31, VADERM31, VAGERD31, VAHRLS31, VABACK31, VAJTPN31, VARTHR31, VAGOUT31, VANECK31, VAFIBR31, VATMD31, VAPTSD31, VALCOH31, VABIPL31, VADEPR31, VAMOOD31, VAPROS31, VARHAB31, VAMNHC31, VAGCNS31, VARXMD31, VACRGV31, VAMOBL31, VACOST31, VARECM31, VAREP31, VAWAIT31, VALOCT31, VANTWK31, VANEED31, VAOUT31, VAPAST31, VACOMP31, VAMREC31, VAGTRC31, VACARC31, VAPROB31, VACARE31, VAPACT31, VAPCPR31, VAPROV31, VAPCOT31, VAPCCO31, VAPCRC31, VAPCSN31, VAPCRF31, VAPCSO31, VAPCOU31, VAPCUN31, VASPCL31, VASPMH31, VASPOU31, VASPUN31, VACMPM31, VACMPY31, VAPROX31, EMPST31, RNDFLG31, MORJOB31, HRWGIM31, HRHOW31, DIFFWG31, NHRWG31, HOUR31, TEMPJB31, SSNLJB31, SELFCM31, CHOIC31, INDCAT31, NUMEMP31, MORE31, UNION31, NWK31, STJBMM31, STJBYY31, OCCCAT31, PAYVAC31, SICPAY31, PAYDR31, RETPLN31, BSNTY31, JOBORG31, OFREMP31, CMJHLD31, MCRPD31, MCRPB31, MCRPHO31, MCDHMO31, MCDMC31, PRVHMO31, FSAGT31, HASFSA31, PFSAMT31, MCAID31, MCARE31, GOVTA31, GOVAAT31, GOVTB31, GOVBAT31, GOVTC31, GOVCAT31, VAPROG31, VAPRAT31, IHS31, IHSAT31, PRIDK31, PRIEU31, PRING31, PRIOG31, PRINEO31, PRIEUO31, PRSTX31, PRIV31, PRIVAT31, VERFLG31, DENTIN31, DNTINS31, PMEDIN31, PMDINS31, PMEDUP31, PMEDPY31, AGE31X, MARRY31X, FTSTU31X, REFRL31X, MOPID31X, DAPID31X, HRWG31X, DISVW31X, HELD31X, OFFER31X, TRIST31X, TRIPR31X, TRIEX31X, TRILI31X, TRICH31X, MCRPD31X, TRICR31X, TRIAT31X, MCAID31X, MCARE31X, MCDAT31X, PUB31X, PUBAT31X, INS31X, INSAT31X, SEX, RACEV1X, RACEV2X, RACEAX, RACEBX, RACEWX, RACETHX, HISPANX, HISPNCAT, EDUCYR, HIDEG, OTHLGSPK, HWELLSPK, BORNUSA, WHTLGSPK, YRSINUS

\paragraph{Demographic features.}
The full list of demographic features we use is:
\begin{itemize}
\item AGE31X
\item SEX
\item RACEV1X
\item RACEV2X
\item RACEAX
\item RACEBX
\item RACEWX
\item RACETHX
\item HISPANX
\item HISPNCAT
\item EDUCYR
\item HIDEG
\item OTHLGSPK
\item HWELLSPK
\item BORNUSA
\item WHTLGSPK
\item YRSINUS
\end{itemize}

\paragraph{Target variable.}
We construct the target variable by summing up the following features:
\begin{itemize}
\item OBTOTV19 --- NUMBER OF OFFICE-BASED PROVIDER VISITS 2019
\item OPTOTV19 --- NUMBER OF OUTPATIENT DEPT PROVIDER VISITS 2019
\item ERTOT19 --- NUMBER OF EMERGENCY ROOM VISITS 2019
\item IPNGTD19 --- NUMBER OF NIGHTS IN HOSP FOR DISCHARGES, 2019
\item HHTOTD19 ---  NUMBER OF HOME HEALTH PROVIDER DAYS 2019
\end{itemize}
We label all instances \emph{positive} ($1$) where the sum is strictly greater than $3$. We label all other instances \emph{negative} ($0$). This leads to 53.17\% positive instances. Hence, an all ones predictor achieves 46.83\% classification error.

\paragraph{Full set of figures.} Figure~\ref{fig:meps-zeroone-full} shows all results for the zero-one loss, Figure~\ref{fig:meps-squared-full} for the squared loss, and Figure~\ref{fig:meps-roc-full} for ROC curves.

\begin{figure}
\includegraphics[width=0.98\linewidth]{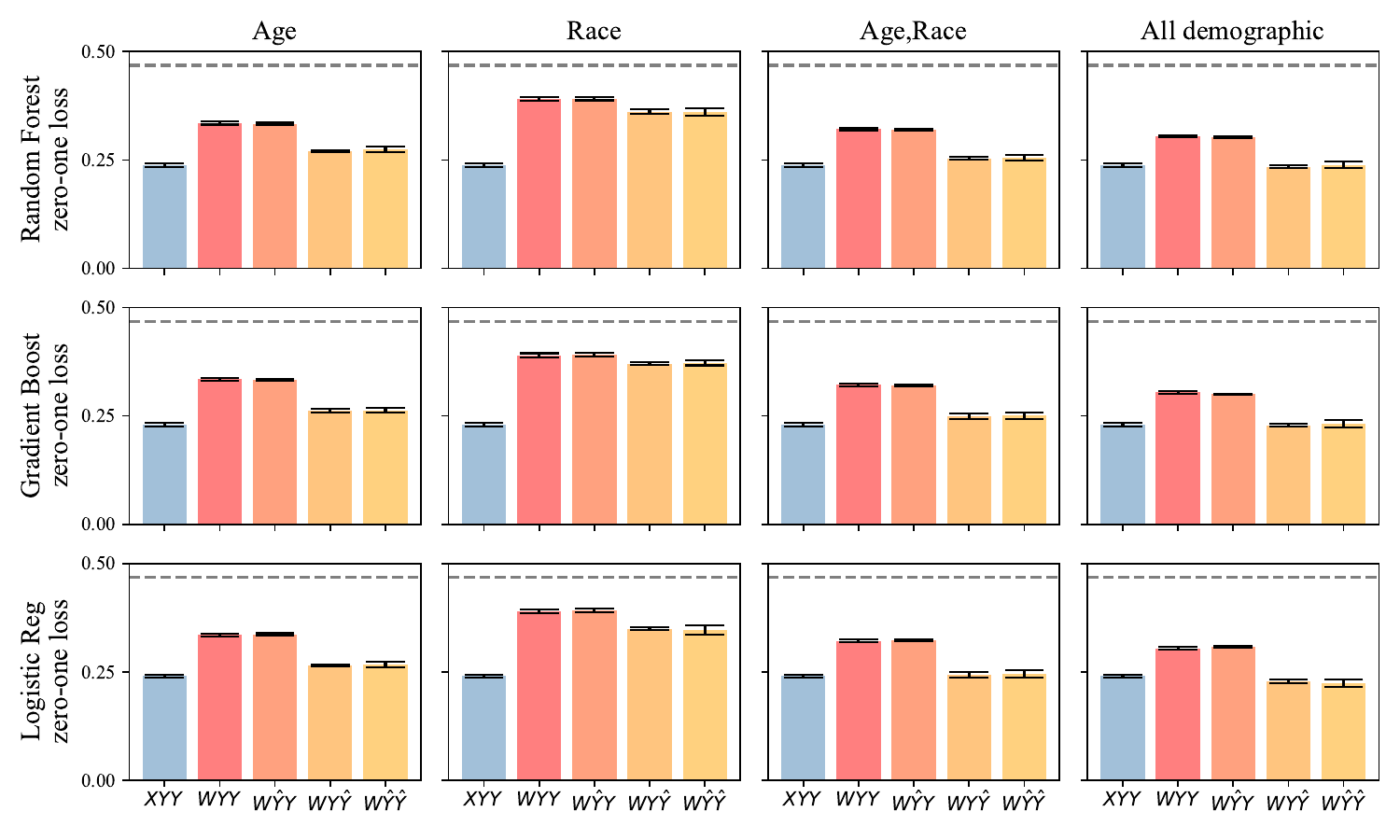}
\caption{Baselines on MEPS for varying features and classifiers (zero-one loss)}
\label{fig:meps-zeroone-full}
\end{figure}

\begin{figure}
\includegraphics[width=0.98\linewidth]{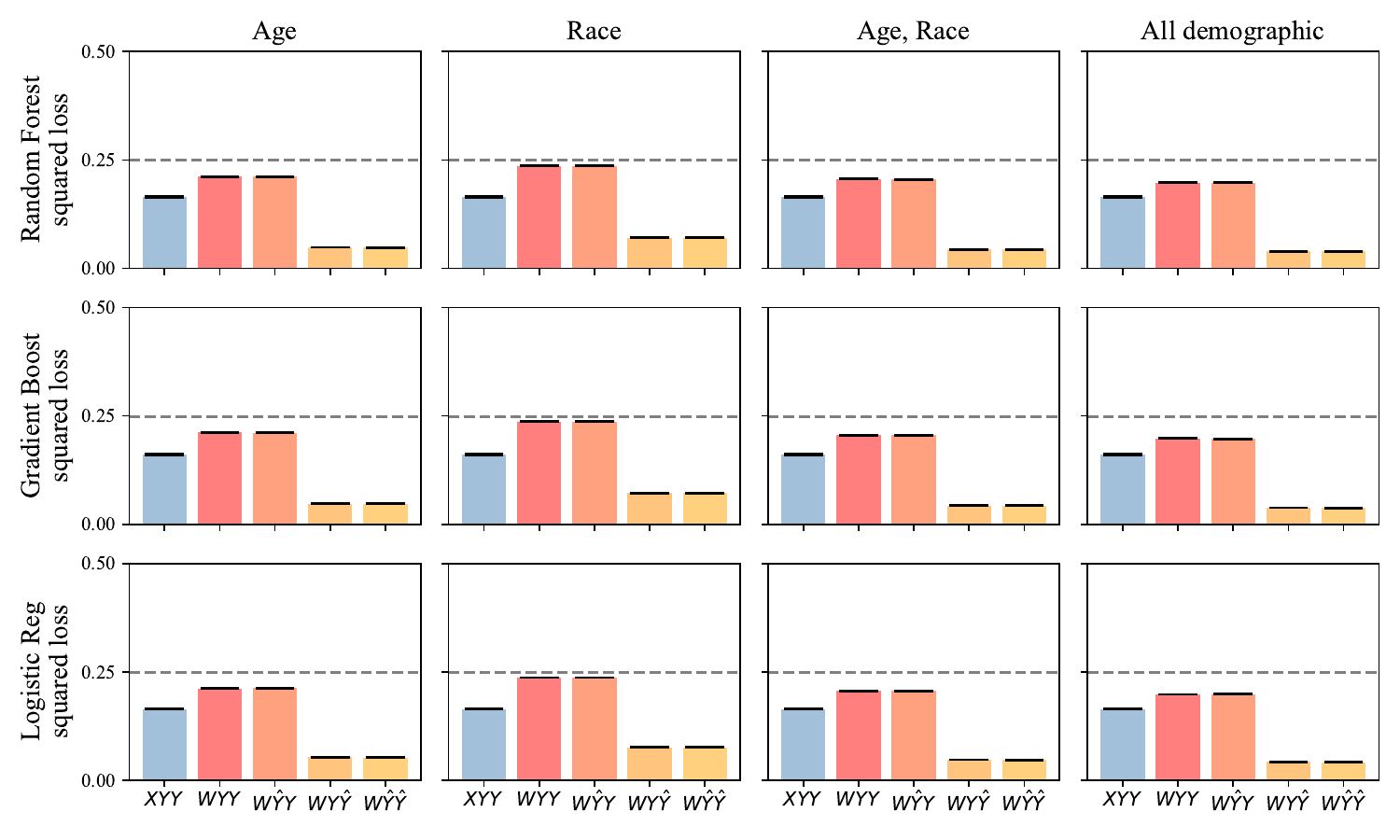}
\caption{Baselines on MEPS for varying features and classifiers (squared loss)}
\label{fig:meps-squared-full}
\end{figure}

\begin{figure}
\includegraphics[width=0.98\linewidth]{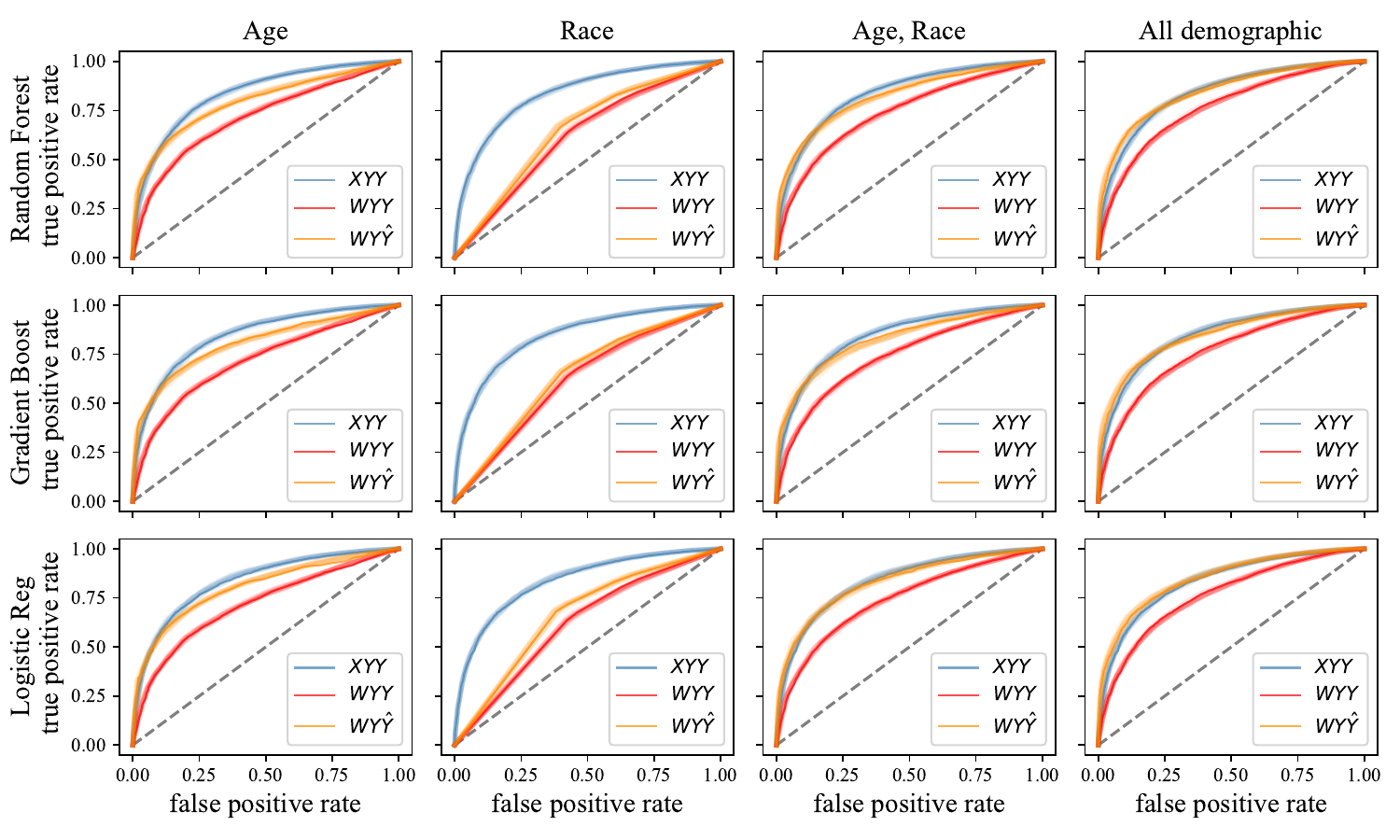}
\caption{Baselines on MEPS for varying features and classifiers (ROC curves)}
\label{fig:meps-roc-full}
\end{figure}

\subsection{Survey of Income and Program Participation (SIPP)}

Extensive documentation and background information on this survey is available from the websites of the US Census Bureau: \url{https://www.census.gov/programs-surveys/sipp.html}

\paragraph{Data availability and conditions.} 
The SIPP data provided by the US Census Bureau are in the public domain. We use the first two waves of the SIPP 2014 Panel data, available here:
\begin{itemize}
    \item Wave 1: \url{https://www.census.gov/programs-surveys/sipp/data/datasets/2014-panel/wave-1.html}
    \item Wave 2: \url{https://www.census.gov/programs-surveys/sipp/data/datasets/2014-panel/wave-2.html}
\end{itemize}

\paragraph{Features.}
The dataset we derive from the 2014 SIPP panel data uses a set of $50$ variables constructed from one or multiple variables appearing in the SIPP raw data in Wave 1. The list below shows each feature we use (in capital letters) followed by the original SIPP feature(s) it is derived from.
\begin{itemize}
\item  LIVING\_QUARTERS\_TYPE : tlivqtr
\item  LIVING\_OWNERSHIP : etenure
\item  SNAP\_ASSISTANCE : efs
\item  WIC\_ASSISTANCE : ewic
\item  MEDICARE\_ASSISTANCE : emc
\item  MEDICAID\_ASSISTANCE : emd
\item  HEALTHDISAB : edisabl
\item  DAYS\_SICK : tdaysick
\item  HOSPITAL\_NIGHTS : thospnit
\item  PRESCRIPTION\_MEDS : epresdrg
\item  VISIT\_DENTIST\_NUM : tvisdent
\item  VISIT\_DOCTOR\_NUM : tvisdoc
\item  HEALTH\_INSURANCE\_PREMIUMS : thipay
\item  HEALTH\_OVER\_THE\_COUNTER\_PRODUCTS\_PAY : totcmdpay
\item  HEALTH\_MEDICAL\_CARE\_PAY : tmdpay
\item  HEALTH\_HEARING : ehearing
\item  HEALTH\_SEEING : eseeing
\item  HEALTH\_COGNITIVE : ecognit
\item  HEALTH\_AMBULATORY : eambulat
\item  HEALTH\_SELF\_CARE : eselfcare
\item  HEALTH\_ERRANDS\_DIFFICULTY : eerrands
\item  HEALTH\_CORE\_DISABILITY : rdis
\item  HEALTH\_SUPPLEMENTAL\_DISABILITY : rdis\_alt
\item  AGE : tage
\item  GENDER : esex
\item  RACE : trace
\item  EDUCATION : eeduc
\item  MARITAL\_STATUS : ems
\item  CITIZENSHIP\_STATUS : ecitizen
\item  FAMILY\_SIZE\_AVG : rfpersons
\item  HOUSEHOLD\_INC : thtotinc
\item  RECEIVED\_WORK\_COMP : ewc\_any
\item  TANF\_ASSISTANCE : etanf
\item  UNEMPLOYMENT\_COMP : eucany
\item  SEVERANCE\_PAY\_PENSION : elmpnow
\item  FOSTER\_CHILD\_CARE\_AMT : tfccamt
\item  CHILD\_SUPPORT\_AMT : tcsamt
\item  ALIMONY\_AMT : taliamt
\item  INCOME\_FROM\_ASSISTANCE : tptrninc, tpscininc, tpothinc
\item  INCOME : tpprpinc, tptotinc
\item  SAVINGS\_INV\_AMOUNT : tirakeoval, tthr401val
\item  UNEMPLOYMENT\_COMP\_AMOUNT : tuc1amt, tuc2amt, tuc3amt
\item  VA\_BENEFITS\_AMOUNT : tva1amt, tva2amt, tva3amt, tva4amt, tva5amt
\item  RETIREMENT\_INCOME\_AMOUNT : tret1amt, tret2amt, tret3amt, tret4amt, tret5amt, tret6amt, tret7amt, tret8amt
\item  SURVIVOR\_INCOME\_AMOUNT : tsur1amt, tsur2amt, tsur3amt, tsur4amt, tsur5amt, tsur6amt, tsur7amt, tsur8amt, tsur11amt, tsur13amt
\item  DISABILITY\_BENEFITS\_AMOUNT : tdis1amt, tdis2amt, tdis3amt, tdis4amt, tdis5amt, tdis6amt, tdis7amt, tdis10amt
\item  FOOD\_ASSISTANCE : efood\_type1, efood\_type2, efood\_type3, efood\_oth
\item  TRANSPORTATION\_ASSISTANCE : etrans\_type1, etrans\_type2, etrans\_type3, etrans\_type4, etrans\_oth
\item  SOCIAL\_SEC\_BENEFITS : esssany, esscany
\end{itemize}
These variables represent features derived from columns in the original data source via our own data cleaning and processing script. In particular, we discount columns that have more than 10\% missing values.

\paragraph{Demographic features.}
The full list of six demographic features we use is:
\begin{itemize}
\item AGE
\item GENDER
\item RACE
\item EDUCATION
\item MARITAL\_STATUS
\item CITIZENSHIP\_STATUS
\end{itemize}

\paragraph{Target variable.}
The target variable is constructed based on the feature thcyincpov in Wave~2, which reflects the household income-to-poverty ratio in the 2019 calendar year, excluding Type~2 individuals. Type~2 individuals are individuals that lives in the household for some month but no longer reside there.

We threshold thcyincpov at 3 so that all instances with thcyincpov strictly greater than 3 are labeled positive ($1$) and all others are labeled negative ($0$). This leads to 51.12\% positive instances. Hence an all ones predictor has accuracy 48.88\%.

\paragraph{Full set of figures.} Figure~\ref{fig:sipp-zeroone-full} shows all results for the zero-one loss, Figure~\ref{fig:sipp-squares-full} for the squared loss, and Figure~\ref{fig:sipp-roc-full} for ROC curves.

\begin{figure}
\includegraphics[width=0.98\linewidth]{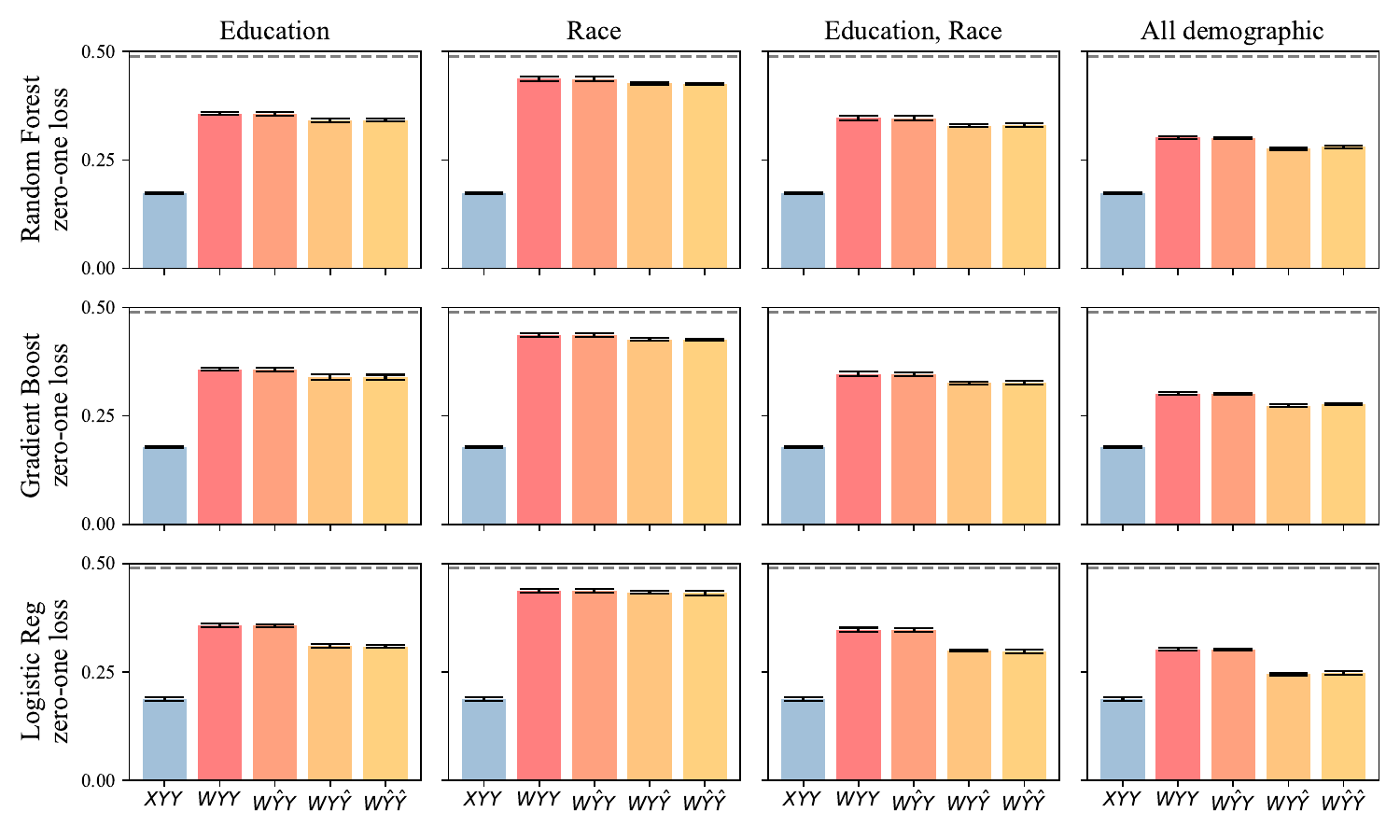}
\caption{Baselines on SIPP for varying features and classifiers (zero-one loss)}
\label{fig:sipp-zeroone-full}
\end{figure}

\begin{figure}
\includegraphics[width=0.98\linewidth]{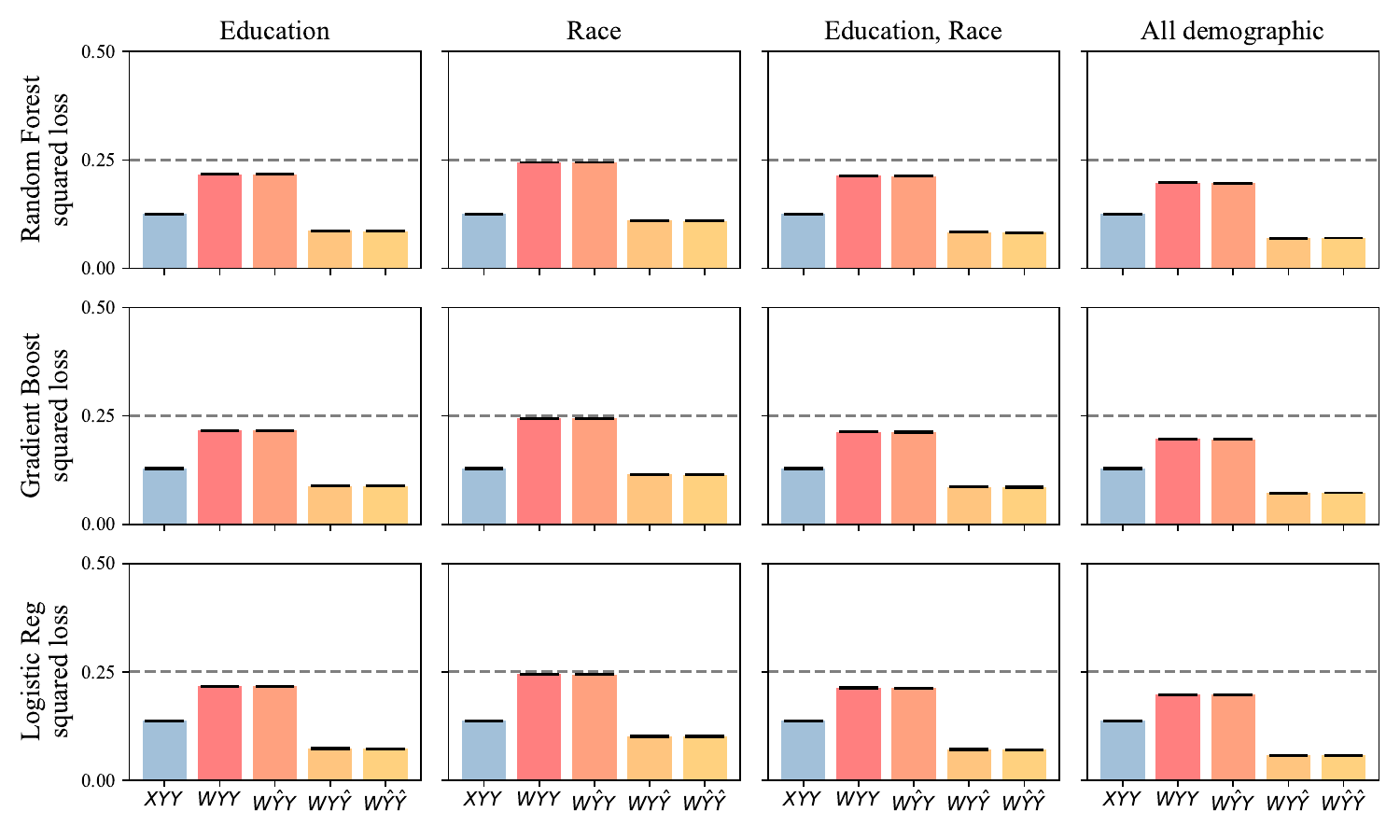}
\caption{Baselines on SIPP for varying features and classifiers (squared loss)}
\label{fig:sipp-squares-full}
\end{figure}

\begin{figure}
\includegraphics[width=0.98\linewidth]{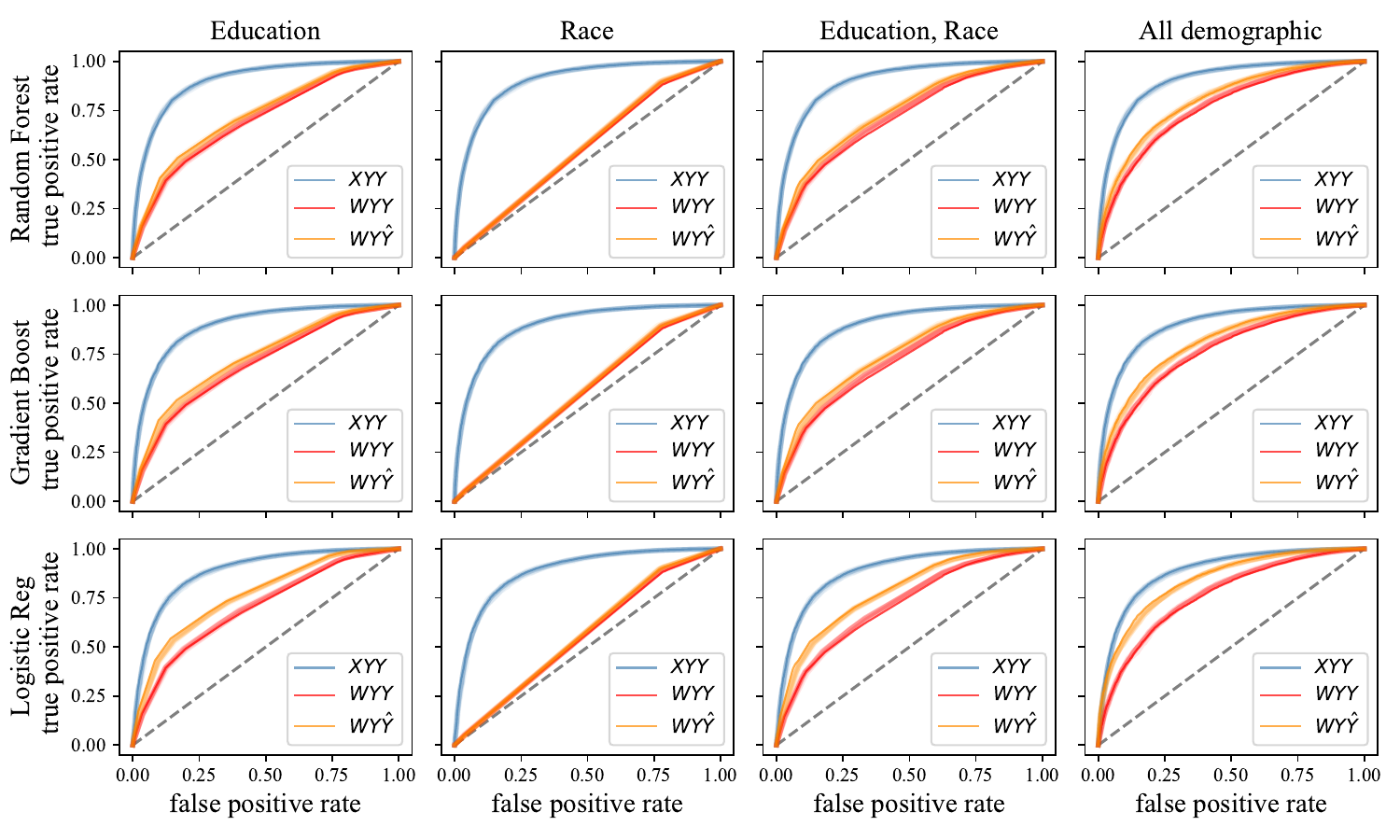}
\caption{Baselines on SIPP for varying features and classifiers (ROC curves)}
\label{fig:sipp-roc-full}
\end{figure}

\subsection{ProPublica COMPAS Recidivism Scores}

\paragraph{Data sources and use conditions.}
We use the COMPAS score dataset collected and made available by Problica~\cite{angwin2016machine}, which is widely used througout the algorithmic fairness literature. The Propublica COMPAS score dataset is available online: \url{https://github.com/propublica/compas-analysis}
The data repository does not specify a license or data use agreement.

\paragraph{Demographic features.}
We use the following demographic features avilable in the dataset:

\begin{itemize}
\item `race',
\item `age',
\item `juv\_fel\_count', `juv\_misd\_count', `juv\_other\_count' : juvenile priors
\item `prior\_count'
\end{itemize}

\paragraph{Target variable.} 
We use \emph{two-year recidivism} ('two\_year\_recid') as the target variable.

\paragraph{Predictor.}
Since we lack training data, we instead audit COMPAS scores as a black-box. 
The column in the data corresponding to COMPAS scores is called `decile\_score' and provides score deciles. To obtain a predictor we fit a single-variable model to predict the target variable from the score deciles. This amounts to a recalibration of the score values to the target variable, ensuring that we obtain the best possible predictor we can from the score deciles.

\paragraph{Full set of figures.} 
Figure~\ref{fig:COMPAS-full} shows all results we report on the COMPAS dataset.

\begin{figure}[ht]
\includegraphics[width=0.98\linewidth]{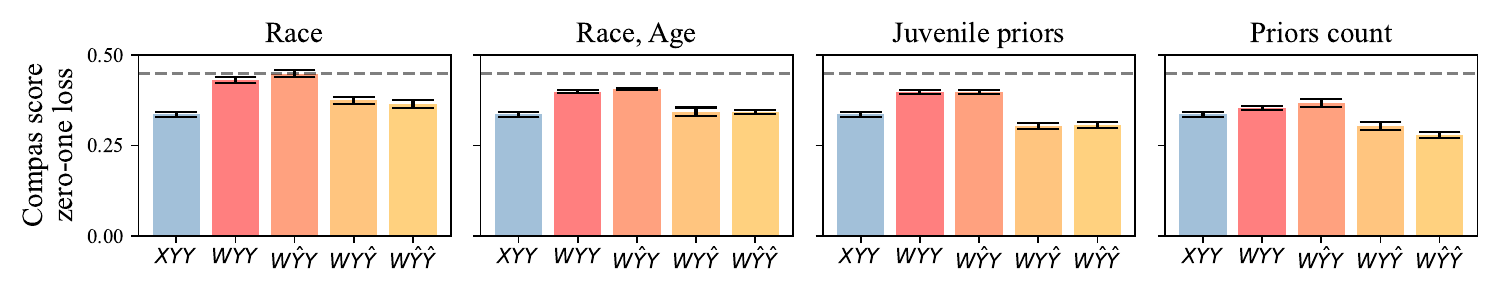}
\includegraphics[width=0.98\linewidth]{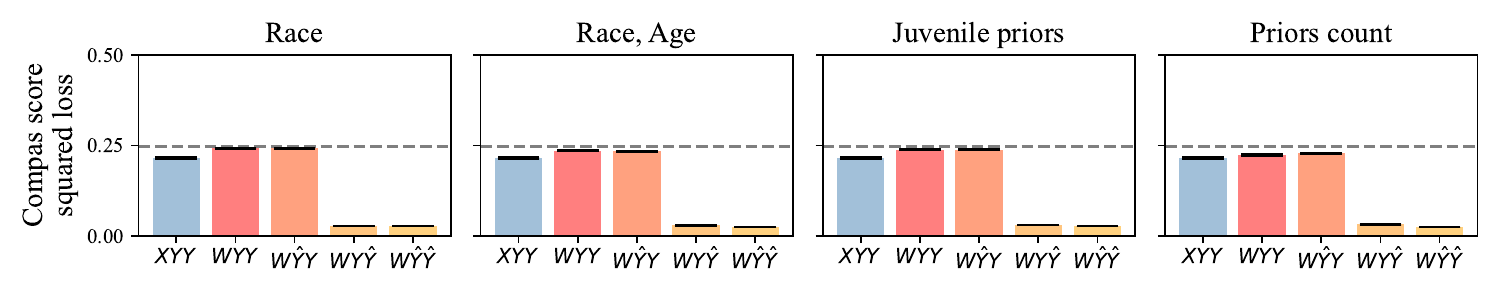}
\includegraphics[width=0.98\linewidth]{results/COMPAS_rocplot.pdf}
\caption{Baselines on COMPAS for varying features and metrics}
\label{fig:COMPAS-full}
\end{figure}

\clearpage
\section{Reference implementation of backward baselines}
\label{sec:code}

\lstset{basicstyle=\footnotesize\ttfamily,breaklines=true}
\begin{lstlisting}[language=Python]
from sklearn.ensemble import GradientBoostingClassifier
from sklearn.metrics import accuracy_score
from sklearn.model_selection import train_test_split


def backward_baselines(X, y, features, model): 
    """Compute backward baselines.

    Parameters
    ----------
    X : numpy.ndarray
        data matrix (n, d)
    y : numpy.ndarray
        target variable (n,)
    features : list
        list of column names
    model : object
        model supporting fit and predict
    
    Returns
    -------
    dict
        Scores of all backward baselines.
    """

    X_train, X_test, y_train, y_test = train_test_split(X, y, test_size=0.33)

    scores = {}
    # XYY
    model.fit(X_train, y_train)
    scores['XYY'] = accuracy_score(model.predict(X_test), y_test)

    # WYY
    baseline = GradientBoostingClassifier()
    baseline.fit(X_train[features], y_train)
    scores['WYY'] = accuracy_score(baseline.predict(X_test[features]), y_test)

    # WY^Y
    baseline.fit(X_test[features], model.predict(X_test))
    scores['WY^Y'] = accuracy_score(baseline.predict(X_test[features]), y_test)

    # WYY^
    baseline.fit(X_test[features], y_test)
    scores['WYY^'] = accuracy_score(baseline.predict(X_test[features]),
                                    model.predict(X_test))

    # WY^Y^ requires new train/test split
    X_testA, X_testB, y_testA, y_testB = train_test_split(X_test[features],
                                              model.predict(X_test), test_size=0.5)
    baseline.fit(X_testA, y_testA)
    scores['WY^Y^'] = accuracy_score(baseline.predict(X_testB), y_testB)

    return scores
\end{lstlisting}

\end{document}